%% file: adapt_alloc_tech-report.tex
 \documentclass[3p]{elsarticle}

\usepackage{epsf}
\usepackage{graphicx}
\usepackage{amsmath}
\usepackage{amssymb}
\usepackage{url}

\usepackage{amsthm}

\usepackage{algorithmic}
\usepackage{algorithm}
\usepackage{epsfig}

\newlength{\minipagewidth}
\newtheorem{lemma}{Lemma}
\newtheorem{theorem}{Theorem}
\newtheorem{proposition}{Proposition}

\newcommand{\bookbox}[1]{\small
\par\medskip\noindent
\framebox[\columnwidth]{
\begin{minipage}{0.8\minipagewidth} {#1} \end{minipage} } \par\medskip }

\input{symbolsdef}
\renewcommand{\N}{\mathcal N}
\renewcommand{\F}{\mathcal F}
\newcommand{\G}{\mathcal G}
\newcommand{\de}{\delta}
\newcommand{\ep}{\epsilon}

\begin{document}
\setlength{\minipagewidth}{\columnwidth}
\addtolength{\minipagewidth}{0.65in}

\begin{frontmatter}

\title{Upper-Confidence-Bound Algorithms for\\ Active Learning in Multi-Armed Bandits}
%
\author{Alexandra Carpentier*, Alessandro Lazaric*, Mohammad Ghavamzadeh*, R{\'e}mi Munos*, Peter Auer **, Andr{\'a}s Antos ***}
%
\address{(*) SequeL team, INRIA Lille - Nord Europe, Team SequeL, France \\
(**) University of Leoben, Franz-Josef-Strasse 18, 8700 Leoben, Austria \\
(***) Budapest University of Technology and Economics, M\H{u}egyetem rkp. 3, 1111 Budapest, Hungary}




\begin{abstract}
In this paper, we study the problem of estimating uniformly well the mean values of several distributions given a finite budget of samples. If the variance of the distributions were known, one could design an optimal sampling strategy by collecting a number of independent samples per distribution that is proportional to their variance. However, in the more realistic case where the distributions are not known in advance, one needs to design adaptive sampling strategies in order to select which distribution to sample from according to the previously observed samples. We describe two strategies based on pulling the distributions a number of times that is proportional to a high-probability upper-confidence-bound on their variance (built from previous observed samples) and report a finite-sample performance analysis on the excess estimation error compared to the optimal allocation. 
We show that the performance of these allocation strategies depends not only on the variances but also on the full shape of the distributions.
\end{abstract}

\end{frontmatter}


\noindent
\textbf{Keywords:} Bandit Theory, Active Learning

\section{Introduction}\label{s:introduction}

Consider a marketing problem where the objective is to estimate the potential impact of several new products or services. A common approach to this problem is to design active online polling systems, where at each time a product is presented (e.g.,~via a web banner on Internet) to random customers from a population of interest, and feedbacks are collected (e.g.,~whether the customer clicks on the ad or not) and used to estimate the average preference of all the products. It is often the case that some products have a general consensus of opinion (low variance) while others have a large variability (high variance). While in the former case very few votes would be enough to have an accurate estimate of the value of the product, in the latter the system should present the product to more customers in order to achieve the same accuracy. Since the variability of the opinions for different products is not known in advance, the objective is to design an active strategy that selects which product to display at each time step in order to estimate the values of all the products uniformly well. 

The problem of online polling can be seen as an online allocation problem with several options, where the accuracy of the estimation of the quality of each option depends on the quantity of the resources allocated to it and also on some (initially unknown) intrinsic variability of the option. This general problem is closely related to the problems of active learning~\citep{cohn1996active,castro2005faster}, sampling and Monte-Carlo methods~\citep{etore2010adaptive}, and optimal experimental design~\citep{fedorov1972theory,chaudhuri1995on-efficient}. A particular instance of this problem is introduced in \cite{antos2010active} as an active learning problem in the framework of stochastic multi-armed bandits. 
More precisely, the problem is modeled as a repeated game between a learner and a stochastic environment, defined by a set of $K$ unknown distributions $\{\nu_k\}_{k=1}^K$, where at each round $t$, the learner selects an action (or arm) $k_t$ and as a consequence receives a random sample from $\nu_{k_t}$ (independent of the past samples). Given a total budget of $n$ samples, the goal is to define an allocation strategy over arms so as to estimate their expected values uniformly well. Note that if the variances $\{\sigma_k^2\}_{k=1}^K$ of the arms were initially known, the optimal allocation strategy would be to sample the arms proportionally to their variances, or more accurately, proportionally to $\lambda_k = \sigma_k^2/ \sum_j \sigma_j^2$. 
However, since the distributions are initially unknown, the learner should follow an active allocation strategy which adapts its behavior as samples are collected. The performance of this strategy is measured by its regret (defined precisely by Equation~\ref{e:regret}) that is the difference between the maximal expected quadratic estimation error of the algorithm and the maximal expected error of the optimal allocation.

\citet{antos2010active} presented an algorithm, called GAFS-MAX, that allocates samples proportionally to the empirical variances of the arms, while imposing that each arm should be pulled at least $\sqrt{n}$ times (to guarantee good estimation of the true variances), where $n$ is the total budget of pulls. They proved that for large enough $n$, the regret of their algorithm scales with $\tilde O(n^{-3/2})$ and conjectured that this rate is optimal.\footnote{The notation $u_n=\tilde O(v_n)$ means that there exist $C>0$ and $\alpha>0$ such that $u_n\leq C (\log n)^\alpha v_n$ for sufficiently large $n$.} However, the performance displays both an implicit (in the condition for large enough $n$) and explicit (in the regret bound) dependency on the inverse of the smallest optimal allocation proportion, i.e.,~$\lambda_{\min} = \min_k \lambda_k$. This suggests that the algorithm is expected to have a poor performance whenever an arm has a very small variance compared to the others. Whether this dependency is due to the analysis of GAFS-MAX, to the specific class of algorithms, or to an intrinsic characteristic of the problem is an interesting open question. One of the main objectives of this paper is to investigate this issue and identify under which conditions this dependency can be avoided. Our main contributions and findings are as follows:
\begin{itemize}
\item We introduce two new algorithms based on upper-confidence-bounds (UCB) on the variance.
\item The first algorithm, called CH-AS, is based on Chernoff-Hoeffding's bound, whose regret has the rate $\tilde O(n^{-3/2})$ and inverse dependency on $\lambda_{\min}$, similar to GAFS-MAX. The main differences are: the bound for CH-AS holds for any $n$ (and not only for large enough $n$), multiplicative constants are made explicit, and finally, the proof is simpler and relies on very simple tools.
 \item The second algorithm, called B-AS, uses a sharper inequality than CH-AS, and has a better performance (in terms of the number of pulls) in targeting the optimal allocation strategy without any dependency on $\lambda_{\min}$. However, moving from the number of pulls to the regret causes the inverse dependency on $\lambda_{\min}$ to appear in the bound again. We show that this might be due to specific shape of the distributions $\{\nu_k\}_{k=1}^K$ and derive a regret bound independent of $\lambda_{\min}$ for the case of Gaussian arms. 
 \item We show empirically that while the performance of CH-AS depends on $\lambda_{\min}$ in the case of Gaussian arms, this dependence does not exist for B-AS and GAFS-MAX, as they perform well in this case. This suggests that {\bf 1)} it is not possible to remove $\lambda_{\min}$ from the regret bound of CH-AS, independent of the arms' distributions, and {\bf 2)} GAFS-MAX's analysis could be improved along the same line as the proof of B-AS for the Gaussian arms. We also report experiments providing insights on the (somehow unexpected) fact that the full shapes of the distributions, and not only their variances, impact the regret of these algorithms.
\end{itemize}

\section{Preliminaries}\label{s:preliminaries}

The allocation problem studied in this paper is formalized as the standard $K$-armed stochastic bandit setting, where each arm $k=1,\ldots,\narms$ is characterized by a distribution $\distro_k$ with mean $\mu_k$ and non--zero variance $\var_k>0$. At each round $t\geq 1$, the learner (algorithm $\alg$) selects an arm $k_t$ and receives a sample drawn from $\nu_{k_t}$ independently of the past. The objective is to estimate the mean values of all the arms uniformly well given a total budget of $n$ pulls. An adaptive algorithm defines its allocation strategy as a function of the samples observed in the past (i.e.,~at time $t$, the selected arm $k_t$ is a function of all the observations up to time $t-1$). After $n$ rounds and observing $T_{k,n}=\sum_{t=1}^{n} \ind{k=k_t}$ samples from each arm $k$, the algorithm $\alg$ returns the empirical estimates $\displaystyle{\hmu_{k,n}=\frac{1}{T_{k,n}}\sum_{t=1}^{T_{k,n}} X_{k,t}}$, where $X_{k,t}$ denotes the sample received when we pull arm $k$ for the $t$-th time. The accuracy of the estimation of each arm $k$ is measured according to its expected squared estimation error, or loss
\begin{equation}\label{e:loss-kn}
L_{k,n} = \expectB{(\distro_i)_{i\leq K}}{\left(\mu_k - \hmu_{k,n}\right)^2}.
\end{equation}
The global performance or loss of  $\alg$ is defined as the worst loss of the arms
\begin{equation}\label{e:global-loss}
L_{n}(\alg) = \max_{1\leq k\leq K} L_{k,n}\;.
\end{equation}


If the variance of the arms were known in advance, one could design an optimal static allocation (i.e., the number of pulls does not depend on the observed samples) by pulling the arms proportionally to their variances. In the case of static allocation, if an arm $k$ is pulled a fixed number of times $T_{k,n}^*$, its loss is computed as\footnote{This equality does not hold when the number of pulls is random, e.g.,~in adaptive algorithms where the strategy depends on the random observed samples.}
\beq \label{eq:loss-kn-fixed}
L_{k,n} = \frac{\var_k}{T_{k,n}^*}\;.
\eeq
By choosing $T_{k,n}^*$ so as to minimize $L_{n}$ under the constraint that $\sum_{k=1}^K T_{k,n}^* = n$, the optimal static allocation strategy $\alg^*$ pulls each arm $k$ (up to rounding effects) $T_{k,n}^* = \frac{\si_k^2n}{\sum_{i=1}^\narms \si_i^2}$ times, and achieves a global performance 
$\label{e:optimal-loss}
L_n(\alg^*) = \Sigma/n,
$
where $\Sigma = \sum_{i=1}^\narms \si_i^2$. We denote by $\lambda_k = \frac{T_{k,n}^*}{n} = \frac{\si_k^2}{ \Sigma}$, the optimal allocation proportion for arm $k$, and by $\lambda_{\min} = \min_{1\leq k\leq K} \lambda_k$, the smallest such proportion.

In our setting where the variances of the arms are not known in advance, the exploration-exploitation trade-off is inevitable: an adaptive algorithm $\alg$ should estimate the variances of the arms (\textit{exploration}) at the same time as it tries to sample the arms proportionally to these estimates (\textit{exploitation}). In order to measure how well the adaptive algorithm $\alg$ performs, we compare its performance to that of the optimal allocation algorithm $\alg^*$, which requires the knowledge of the variances of the arms. For this purpose, we define the notion of \textit{regret} of an adaptive algorithm $\alg$ as the difference between its loss $L_n(\alg)$ and the optimal loss $L_n(\alg^*)$, i.e.,
\begin{equation}\label{e:regret}
R_n(\alg) = L_n(\alg) - L_n(\alg^*).
\end{equation}
It is important to note that unlike the standard multi-armed bandit problems, we do not consider the notion of cumulative regret, and instead, use the excess-loss suffered by the algorithm at the end of the $n$ rounds. This notion of regret is closely related to the \textit{pure exploration} setting (e.g.,~\cite{audibert2010best,bubeck2011pure}). An interesting feature that is shared between this setting and the problem of active learning considered in this paper is that good strategies should play all the arms as a linear function of $n$. This is in contrast with the standard stochastic bandit setting, at which the sub-optimal arms should be played logarithmically in $n$.

In~\citep{antos2010active}, the authors provide an algorithm called GAFS-MAX and they prove that its regret is such that $R_n(\alg_{GAFS-MAX}) = \tilde O(n^{-3/2})$ for a large enough budget $n$ that depends on $\lambda_{\min}$. Also, the $\tilde O$ depends on $\lambda_{\min}$. The smaller $\lambda_{\min}$, the larger $n$ needs to be so that the bound in $\tilde O(n^{-3/2})$ holds, and also the larger the constant in the $\tilde O$.


\section{Allocation Strategy Based on Chernoff-Hoeffding UCB}\label{s:ch-algorithm}

The first algorithm, called {\em Chernoff-Hoeffding Allocation Strategy} (CH-AS), is based on a Chernoff-Hoeffding high-probability bound on the difference between the estimated and true variances of the arms. Each arm is simply pulled proportionally to an upper-confidence-bound (UCB) on its variance. This algorithm deals with the exploration-exploitation trade-off by pulling more the arms with higher estimated variances or higher uncertainty in these estimates. 


\subsection{The CH-AS Algorithm}\label{ss:ch-algorithm}

%
%
The CH-AS algorithm $\alg_{CH}$ in Fig.~\ref{f:ch-algorithm} takes a confidence parameter $\de$ as input and after $n$ pulls returns an empirical mean $\hmu_{k,n}$ for each arm $k$. At each time step $t$, i.e., after having pulled arm $k_t$, the algorithm computes the empirical mean $\hmu_{k,t}$ and variance $\hsi^2_{k,t}$ of each arm $k$ as\footnote{Notice that this is a biased estimator of the variance even if the numbers of pulls $T_{k,t}$ were not random.}
\begin{equation}\label{eq:estim-var1}
\hmu_{k,t}=\frac{1}{T_{k,t}}\sum_{i=1}^{T_{k,t}}X_{k,i} \enspace\enspace\enspace\text{ and }\enspace\enspace\enspace\hsi_{k,t}^2 = \frac{1}{T_{k,t}} \sum_{i=1}^{T_{k,t}}X_{k,i}^2-\hmu_{k,t}^2\;,
\end{equation}
%
where $X_{k,i}$ is the $i$-th sample of $\distro_k$ and $T_{k,t}$ is the number of pulls\footnote{An accurate notation for this should be $T_{k,t,n}$ since the number of pulls at time $t$ depends also on $n$. However, for the sake of concision, we note $T_{k,t}$.} allocated to arm $k$ up to time $t$. After pulling each arm twice (rounds $t=1$ to $2K$), from round $t=2K+1$ on, the algorithm computes the $B_{k,t}$ values based on a Chernoff-Hoeffding's bound on the variances of the arms:
%
\begin{equation*}
B_{k,t} = \frac1{T_{k,t-1}} \Big( \hsi_{k,t-1}^2 + 3\sqrt{\frac{\log(1/\de)}{2T_{k,t-1}}}\Big),
\end{equation*}
%
\noindent and then pulls the arm $k_t$ with the largest $B_{k,t}$. This bound relies on the assumption that the distributions $\{\nu_k\}_{k=1}^K$ are supported $[0,1]$.

Note that actually $\hmu_{k,t}$, $\hsi_{k,t}$, $B_{k,t}$, $k_t$, and $T_{k,t}$ depend on the arm index (except for $k_t$), on the time step $t\le n$, but also, either in a direct or in an indirect way (through the mechanism of the algorithm) on the budget $n$ and on $\de$ which will be chosen as a function of the budget $n$. However, since we consider most of the time a fixed budget $n$ and thus a fixed $\de$, we conserve this notation in order to have lighter notations.

\begin{figure}[t]
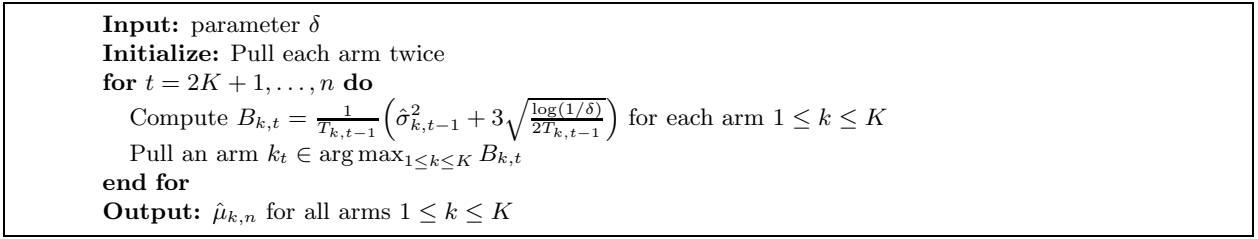

\bookbox{
\begin{algorithmic}
\STATE \textbf{Input:} parameter $\de$
\STATE \textbf{Initialize:} Pull each arm twice
\FOR{$t = 2\narms+1,\ldots, n$}
  \STATE Compute $B_{k,t} = \frac1{T_{k,t-1}} \Big( \hsi_{k,t-1}^2 + 3\sqrt{\frac{\log(1/\de)}{2T_{k,t-1}}}\Big)$ for each arm $1\leq k\leq K$
  \STATE Pull an arm $k_t \in\argmax_{1\leq k\leq K} B_{k,t}$ 
\ENDFOR
\STATE \textbf{Output:} $\hmu_{k,n}$ for all arms $1\leq k\leq K$
\end{algorithmic}
}
\caption{The pseudo-code of the CH-AS algorithm, with $\hsi_{k,t}^2$ computed as in Equation~\ref{eq:estim-var1}.}\label{f:ch-algorithm}
\end{figure}


\subsection{Regret Bound and Discussion}\label{ss:ch-discussion}

Before reporting a regret bound for the CH-AS algorithm, we first analyze its performance in targeting the optimal allocation strategy in terms of the number of pulls. As it will be discussed later, the distinction between the performance in terms of the number of pulls and the regret will allow us to stress the potential dependency of the regret on the distribution of the arms (see Section~\ref{sec:gaussian.dist}).

\begin{lemma}\label{l:ch-upper-lower}
Assume that the distributions $\{\nu_k\}_{k=1}^K$ are supported on $[0,1]$ and let $\de>0$. Define the event
\begin{equation*}
\xi_{K,n}^{CH}(\de) = \mathop{\bigcap_{1\leq k\leq K}}_{1\leq t\leq n}\left\lbrace \Big|\Big(\frac{1}{t}\sum_{i=1}^tX_{k,i}^2-\big(\frac{1}{t}\sum_{i=1}^tX_{k,i}\big)^2\Big) - \si_k^2\Big| \leq 3 \sqrt{\frac{\log\invdelta}{2t}} \right\rbrace.
\end{equation*}
The probability of $\xi_{K,n}^{CH}(\de)$ is higher than or equal to $1-4nK\de$. If $n\geq 5K$, the number of pulls $T_{k,n}$ by the CH-AS algorithm launched with parameter $\de$ satisfies on $\xi_{K,n}^{CH}(\de)$
\begin{equation}\label{eq:up-low-bound}
- \lambda_k \Big(\frac{12\sqrt{n\log\invdelta}}{\Sigma\lambda_{\min}^{3/2}} + 4 K\Big) \leq T_{k,n}- T_{k,n}^* \leq \frac{12\sqrt{n\log\invdelta}}{\Sigma\lambda_{\min}^{3/2}} + 4K ,
\end{equation}
for any arm $1\leq k\leq K$.
\end{lemma}

\begin{proof}
The proof is reported in~\ref{a:l:ch-upper-lower}.
\end{proof}

We now show how the bound on the number of pulls translates into a regret bound for the CH-AS algorithm.

\begin{theorem}\label{thm:ch-regret}
Assume that the distributions $\{\nu_k\}_{k=1}^K$ are supported on $[0,1]$. If the fixed (known in advance) budget is such that $n\geq 5K$, the regret of $\alg_{CH}$, when it runs with the parameter $\de= n^{-5/2}$, is bounded as
\begin{equation}\label{eq:ch-regret}
R_n(\alg_{CH}) \leq 
\frac{39\sqrt{\log(n)}}{n^{3/2} \lambda_{\min}^{5/2}} +  \frac{2.9\times 10^3}{n^2}\frac{(\log n)^{3/2}}{\lambda_{\min}^{11/2}}\Big(1 + \frac{1}{\Sigma^{5/2}}\Big).
\end{equation}
\end{theorem}

\begin{proof}
The proof is reported in~\ref{a:thm:ch-regret}. It is mainly based on the last lemma and on the following inequality (Equation~\ref{eq:st2.1}):
\begin{align}
\E\Big[(\hmu_{k,n} - \mu_k)^2 \1\{{\xi\}}\Big] &\leq  \sup_{\xi}\Big(\frac{\si_k^2}{T_{k,n}^2} \Big) \E[T_{k,n}]\;. \nonumber
\end{align}
\end{proof}

\paragraph{Remark~1} As discussed in Section~\ref{s:preliminaries}, our objective is to design a sampling strategy capable of estimating the mean values of the arms almost as accurately as the estimations by the optimal allocation strategy, which assumes that the variances of the arms are known. In fact, Theorem~\ref{thm:ch-regret} shows that the CH-AS algorithm provides a uniformly accurate estimation of the expected values of the arms with a regret $R_n(\alg_{CH})$ of order $\tilde O(n^{-3/2})$. This regret rate is the same as the one for the GAFS-MAX algorithm in~\citet{antos2010active}. Note also that this algorithm is efficient for a fixed horizon $n$, although it might be possible to change it so that it is efficient for any horizon.

\paragraph{Remark~2} 
The bound displays an inverse dependency on the smallest optimal allocation proportion $\lambda_{\min}$. As a result, the bound scales poorly when an arm has a very small variance relative to the others, i.e.,~$\sigma_k \ll \Sigma$. Note that GAFS-MAX (see~\cite{antos2010active}) has also a similar dependency on the inverse of $\lambda_{\min}$. Moreover, Theorem~\ref{thm:ch-regret} holds for a budget $n\geq 5K$, whereas the regret bound of GAFS-MAX in~\cite{antos2010active} requires a condition $n\geq n_0$, in which $n_0$ is a constant that scales with $1/\lambda_{\min}$. Finally, note that this UCB type of algorithm (CH-AS) enables a much simpler regret analysis than that of GAFS-MAX.

\paragraph{Remark~3} It is clear from Lemma~\ref{l:ch-upper-lower} that the inverse dependency on $\lambda_{\min}$ appears in the bound on the number of pulls and then is propagated to the regret bound. We however believe that this dependency is not an artifact of the analysis and is intrinsic in the performance of the algorithm. Let us consider a two-arm problem with $\sigma_1^2=1/4$ and $\sigma_2^2 = 0$. The optimal allocation is $T_{1,n}^* = n-1$, $T_{2,n}^* = 1$ (only one sample is enough to estimate the mean of the second arm), and $\lambda_{\min}=0$. In this case, the arguments used in proving Theorem~\ref{thm:ch-regret} do not hold anymore and the bound itself becomes vacuous. We conjecture that the Chernoff-Hoeffding's bound used in the upper-confidence term forces the CH-AS to pull the arm with zero variance at least $Dn^{2/3}$ times, where $D$ is a positive constant, with high probability, which results in under-pulling the first arm by the same amount. As a result, the corresponding regret would have a rate of $n^{-4/3}$ w.r.t. the budget $n$. This suggests that when $\lambda_{\min}=0$ (or very small compared to $1/n$) CH-AS is still able to achieve a $o(1/n)$ regret as the budget $n$ increases but with a slower rate w.r.t. to result proved in Theorem~\ref{thm:ch-regret}. 

Finally, we notice that, for $\lambda_{\min}=0$, GAFS-MAX is more efficient than CH-AS. In fact, it over-pulls the arms with zero-variance only by $O(n^{1/2})$ and has a regret of order $\tilde O(n^{-3/2})$.
We will further study how the regret of CH-AS changes with $n$ in Section~\ref{ss:CH-B-G}.

As discussed in the previous remark, the reason for the poor performance in Lemma~\ref{l:ch-upper-lower} for small $\lambda_{\min}$ can be identified in the fact that Chernoff-Hoeffding's inequality is not tight for small-variance random variables. In Section~\ref{s:b-algorithm}, we propose an algorithm based on a tighter inequality for small-variance random variables, and prove that this algorithm under-pulls all the arms by \textit{at most} $\tilde O(n^{1/2})$, without a dependency on $\lambda_{\min}$ (see Equations~\ref{eq:b-lower} and~\ref{eq:b-upper}). 



\section{Allocation Strategy Based on Bernstein UCB}\label{s:b-algorithm}

In this section, we present another UCB-like algorithm, called {\em Bernstein Allocation Strategy} (B-AS)%
\footnote{The original Bernstein inequality refines the Chernoff-Hoeffding's inequality by introducing the variance of the random variable in the confidence bound. This inequality has been later adapted to the case where the actual variance is unknown and it can be replaced by an empirical estimate of the variance (see~\cite{AudibertTCS09}). In~\cite{maurer2009empirical} a similar result is obtained for the variance, where the confidence bound displays a dependency on the empirical estimate of the variance, thus we refer to this algorithm as Bernstein Allocation Strategy. Furthermore, we notice that the inequality derived in~\cite{maurer2009empirical} does not follow from a trivial application of Chernoff-Hoeffding, since it provides a concentration inequality for the standard deviation which is not an average of i.i.d. random variables but the square root of an average of squared variables.}%
, based on a tighter variance confidence bound that enables us to improve the bound on $|T_{k,n}-T_{k,n}^*|$ by removing the inverse dependency on $\lambda_{\min}$ (compare the bounds in Equations~\ref{eq:b-lower} and~\ref{eq:b-upper} to the one for CH-AS in Equation~\ref{eq:up-low-bound}). However this result itself is not sufficient to derive a better regret bound than CH-AS. This finding is interesting since it shows that even an adaptive algorithm which implements a strategy close to the optimal allocation strategy may still incur a regret that poorly scales with the smallest proportion $\lambda_{\min}$. We further investigate this issue by showing that the way the bound on the number of pulls translates into a regret bound depends on the specific distributions of the arms. In fact, when the distributions of the arms are Gaussian, we can exploit the property that the empirical variance $\hsi^2_{k,t}$ is independent of the empirical mean $\hmu_{k,t}$, and show that the regret of B-AS no longer depends on $1/\lambda_{\min}$. The numerical simulations in Section~\ref{s:experiments} further illustrate how the full shape of the distributions (and not only their first two moments) plays an important role in the regret of adaptive allocation algorithms.



\subsection{The B-AS Algorithm}\label{ss:b-algorithm}

\begin{figure}[t]
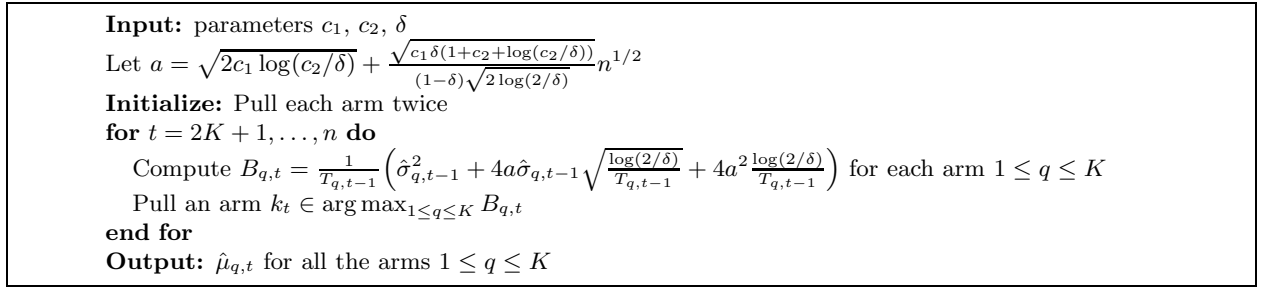

\bookbox{
\begin{algorithmic}
\STATE \textbf{Input:} parameters $c_1$, $c_2$, $\de$
\STATE Let $a = \sqrt{2c_1 \log(c_2/\de)} + \frac{\sqrt{c_1 \de (1+c_2 +\log(c_2/\de))}}{(1-\de)\sqrt{2\log(2/\de)}}n^{1/2}$
\STATE \textbf{Initialize:} Pull each arm twice
\FOR{$t = 2\narms+1,\ldots,n$}
  \STATE Compute $B_{q,t} = \frac1{T_{q,t-1}} \Big( \hvar_{q,t-1} + 4a \hsi_{q,t-1} \sqrt{\frac{\log(2/\de)}{T_{q,t-1}}} + 4a^2\frac{\log(2/\de)}{T_{q,t-1}} \Big)$ for each arm $1\leq q\leq K$
  \STATE Pull an arm $k_t \in \argmax_{1\leq q\leq K} B_{q,t}$ 
\ENDFOR
\STATE \textbf{Output:} $\hmu_{q,t}$ for all the arms $1\leq q\leq K$
\end{algorithmic}}
\caption{The pseudo-code of the B-AS algorithm. The empirical variances $\hsi^2_{k,t}$ are computed according to Equation~\ref{eq:estim-var2}.}\label{f:b-algorithm}
\end{figure}

The algorithm is based on the use of a high-probability bound, reported in \cite{maurer2009empirical} (a similar bound can be found in \cite{AudibertTCS09}), on the variance of each arm. Like in the previous section, the arm sampling strategy is determined by those bounds. The B-AS algorithm, $\alg_B$, is described in Figure~\ref{f:b-algorithm}. It requires three parameters as input (see Remark~2 in Subsection~\ref{sec:gaussian.dist} for a discussion on how to reduce the number of parameters from three to one) $c_1$ and $c_2$, which are related to the shape of the distributions (see Assumption~\ref{a:subgaussian}), and $\de$, which defines the \textit{confidence level} of the bound. The amount of exploration of the algorithm can be adapted by properly tuning these parameters. The algorithm is similar to CH-AS except that for each arm, the bound $B_{q,t}$ is computed as
%
\begin{equation*}
B_{q,t} = \frac1{T_{q,t-1}} \Big( \hvar_{q,t-1} + 4a \hsi_{q,t-1} \sqrt{\frac{\log(2/\de)}{T_{q,t-1}}} + 4a^2\frac{\log(2/\de)}{T_{q,t-1}} \Big)\;,
\end{equation*}
%
\noindent
where $a = \sqrt{2c_1 \log(c_2/\de)} + \frac{\sqrt{c_1 \de (1+c_2 +\log(c_2/\de))}}{(1-\de)\sqrt{2\log(2/\de)}}n^{1/2}$, 
and\footnote{Unlike in Equation~\ref{eq:estim-var1}, here we use the unbiased estimator of variance.}

\begin{equation}\label{eq:estim-var2}
\hat\mu_{k,t} = \frac{1}{T_{k,t}} \sum_{i=1}^{T_{k,t}} X_{k,i},\;\;\;\;\;\mbox{ and }
\;\;\;\;\;\hsi_{k,t}^2 = \frac{1}{T_{k,t}-1} \sum_{i=1}^{T_{k,t}} (X_{k,i}-\hat\mu_{k,t})^2\;.
\end{equation}

Note that actually $\hmu_{k,t}$, $\hsi_{k,t}$, $B_{k,t}$, $k_t$, and $T_{k,t}$ depend on the arm index (except for $k_t$), on the time step $t\le n$, but also, either in a direct or in an indirect way (through the mechanism of the algorithm) on the budget $n$, on $\de$ which will be chosen as a function of the budget $n$, and also on $c_1$ and $c_2$. However, since we consider most of the time a fixed budget $n$ and thus a fixed $\de$, and fixed $c_1,c_2$, we conserve this notation in order to have lighter notations.


\subsection{Regret Bound and Discussion}\label{s:b-discussion}

The B-AS algorithm is designed to overcome the limitations of CH-AS, especially in the case of arms with different variances. Here we consider a more general assumption than in the previous section, namely that the distributions are sub-Gaussian.

\begin{assumption}[Sub-Gaussian distributions]\label{a:subgaussian}
There exist $c_1,c_2>0$ such that for all $1\leq k\leq K$ and any $\ep>0$, 
\begin{equation}\label{def:sub.gaus}
\Prob_{X\sim \nu_k}[|X-\mu_k| \geq \ep] \leq c_2 \exp(-\ep^2/c_1)\;.
\end{equation}
\end{assumption}

This assumption holds for the Gaussian distribution, and more generally for any distribution whose tail is lighter than Gaussian's. It is thus held for bounded random variables. For example, if $X \in [0,1]$, then the assumption holds with e.g.,~$c_1 = 1$ and $c_2 = e$.

We first state a bound in Lemma~\ref{l:b-upper-lower} on the difference between the number of pulls suggested by B-AS and the optimal allocation strategy.

\begin{lemma}\label{l:b-upper-lower}
Let Assumption~\ref{a:subgaussian} holds for $c_1,c_2\geq 1$ and let $0<\de\le 2/e$. Define the event
\begin{equation*}
\xi_{K,n}^B(\de) = \mathop{\bigcap_{1\leq k\leq K}}_{2\leq t\leq n}\left\lbrace \Bigg|\sqrt{\frac{1}{t-1}\sum_{i=1}^t\Big(X_{k,i} - \frac{1}{t}\sum_{j=1}^t X_{k,j}\Big)^2} - \si_k\Bigg| \leq 2a\sqrt{\frac{\log(2/\de)}{t}} \right\rbrace,
\end{equation*}
where $a =  \sqrt{2c_1 \log(c_2/\de)} + \frac{\sqrt{c_1 \de (1+c_2 +\log(c_2/\de))}}{(1-\de)\sqrt{2\log(2/\de)}}n^{1/2}$. 
The probability of $\xi_{K,n}^B(\de)$ is higher than $1-2nK\de$. When we run the B-AS algorithm with parameters $c_1\geq 1$, $c_2\geq 1$, and $\de$, and budget $n\geq 5K$, on $\xi_{K,n}^B(\de)$ and for each arm $1\leq k\leq K$, we have 
\begin{equation}\label{eq:b-lower}
T_{k,n} \geq T_{k,n}^* - 
K\lambda_k \Bigg[ \frac{16a\sqrt{\log(2/\de)}}{\Sigma} \Bigg(\sqrt{\Sigma} + \frac{2a\sqrt{\log(2/\de)}}{c(\de)} \Bigg)n^{1/2} + 64\sqrt{2K}a^2 \frac{\log(2/\de)}{\Sigma\sqrt{c(\de)}}\;n^{1/4} + 2\Bigg],
\end{equation}
and
\begin{equation}\label{eq:b-upper}
T_{k,n} \leq T_{k,n}^* + 
K \Bigg[ \frac{16a\sqrt{\log(2/\de)}}{\Sigma} \Bigg(\sqrt{\Sigma} + \frac{2a\sqrt{\log(2/\de)}}{c(\de)} \Bigg)n^{1/2} + 64\sqrt{2K}a^2 \frac{\log(2/\de)}{\Sigma\sqrt{c(\de)}}\;n^{1/4} + 2\Bigg],
\end{equation}
where $c(\de) = \frac{ a \sqrt{3\log(2/\de)} }{\sqrt{K}(\sqrt{\Sigma} + 3 a \sqrt{\log(2/\de)} )}$.
\end{lemma}

\begin{proof}
The proof is reported in~\ref{s:b-tools} and~\ref{s:b-allocation}.
\end{proof}

\paragraph{Remark} Unlike the bounds for CH-AS in Lemma~\ref{l:ch-upper-lower}, B-AS allocates the pulls on the arms so that, on the event $\xi_{K,n}^B(\de)$, the bound on the difference between $T_{k,n}$ and $T_{k,n}^*$ is now independent from $\lambda_{\min}$, while it preserves a $\sqrt{n}$ dependency on the budget. In practice, this difference may correspond to a significant improvement. In fact, for any finite budget $n$, if the arms are such that the term depending on $\lambda_{\min}$ becomes the leading term in the bound in Lemma~\ref{l:ch-upper-lower}, then we can expect B-AS to outperform CH-AS (see also Remark 3 of Section~\ref{ss:ch-discussion} for further discussion of the performance of CH-AS for very small $\lambda_{\min}$). Another interesting aspect of the previous lemma is that the lower bound in Equation~\ref{eq:b-lower} can be written as $C \lambda_k \sqrt{n}$ (where $C>0$ does not depend on $\lambda_{k}$). This implies that as allocation ratio $\lambda_k$ decreases (i.e., arm $k$ should not be pulled much), the difference between $T_{k,n}$ and $T_{k,n}^*$ decreases as well. This is not the case in the upper bound, where the difference between $T_{k,n}$ and $T_{k,n}^*$ does not have any linear dependency on $\lambda_k$. This asymmetry between lower and upper bound is the main reason why the final regret bound of B-AS actually displays an inverse dependency on $\lambda_{\min}$ as shown in Theorem~\ref{thm:b-regret}.


\begin{theorem}\label{thm:b-regret}
Assume that all the distributions $\{\nu_k\}_{k=1}^K$ are sub-Gaussians with parameters $c_1$ and $c_2$. If the fixed (known in advance) budget is such that $n\geq 5K$, the regret of $\alg_{B}$, when it runs with parameters $c_1\geq 1$, $c_2 \geq 1$, and $\de= n^{-7/2}$ is bounded as
\begin{align*}
 R_n(\alg_{B}) &\leq 
\frac{76400 c_1(c_2+1) K^2(\log n)^2}{\lambda_{\min}n^{3/2}} + O\Big(\frac{(\log n)^6 K^7}{n^{7/4}\lambda_{\min}}\Big)\;.
\end{align*}
\end{theorem}

\begin{proof}
The proof is reported in~\ref{s:b-regret}.
\end{proof}

Note again that this algorithm is efficient for a fixed horizon $n$, although it might be possible to change it so that it is efficient for any horizon.

Similar to Theorem~\ref{thm:ch-regret}, the bound on the number of pulls translates into a regret bound through Equation~\ref{eq:st2.1} reported in~\ref{a:thm:ch-regret}. Note that in order to remove the dependency on $\lambda_{\min}$, a symmetric bound on $|T_{k,n}-T^*_{k,n}| \leq \lambda_k \tilde O( \sqrt{n})$ is needed. While the lower bound in Equation~\ref{eq:b-lower} already decreases with $\lambda_k$, the upper bound scales with $\tilde O(\sqrt{n})$. Whether there exists an algorithm with a tighter upper bound scaling with $\lambda_k$ is still an open question. Nonetheless, in the next section, we show that an improved bound on the loss can be achieved in the special case of Gaussian distributions, which leads to a regret bound without the dependency on $\lambda_{\min}$.


\subsection{Regret for Gaussian Distributions}\label{sec:gaussian.dist}

In the case of Gaussian distributions, the bound on the loss of Equation~\ref{eq:st2.1} can be improved using the following lemma.

\begin{lemma}\label{l:gauss-regret}
Let $k \leq K$.
Assume that the distribution $\nu_k$ is Gaussian (and independent of all other distributions $(\nu_{k'})_{k' \neq k}$). Then the loss for arm $k$ of algorithms CH-AS or B-AS satisfies
\begin{equation}\label{eq:gauss-regret}
L_{k,n} = \E\big[(\hmu_{k,n} - \mu_k)^2 \big] = \var_k \E\Big[\frac{1}{T_{k,n}} \Big]\;.
\end{equation}
\end{lemma}

\begin{proof}
The proof is reported in~\ref{s:b-results-gauss}.
\end{proof}


\paragraph{Remark} Note that the loss in Equation~\ref{eq:gauss-regret} does not require any upper bound on $T_{k,n}$. It is actually similar to the case of deterministic allocation. When $\tilde T_{k,n}$ is the deterministic number of pulls, the corresponding loss resulting from pulling arm $k$, $\tilde T_{k,n}$ times, is $L_{k,n} = \sigma^2_k / \tilde T_{k,n}$. In general, when $T_{k,n}$ is a random variable depending on the empirical variances $\{\hsi_k^2\}_{k=1}^K$ (like in our adaptive algorithms CH-AS and B-AS), we have
\begin{equation*}
\E\big[(\hmu_{k,n} - \mu_k)^2 \big] = \sum_{t=1}^n \E\big[(\hmu_{k,n} - \mu_k)^2 |T_{k,n}=t\big] \Prob[T_{k,n} = t], 
\end{equation*}
which might be different than $\sigma^2_k \E\Big[\frac{1}{T_{k,n}} \Big]$. In fact, the empirical average $\hmu_{k,n}$ depends on $T_{k,n}$ through $\{\hsi_{k,n}\}_{k=1}^K$, and $\E\big[(\hmu_{k,n} - \mu_k)^2 |T_{k,n}=t\big]$ might not be equal to $\sigma_k^2/t$. However, Gaussian distributions have the property that for any fixed-size sample, the empirical mean is independent from the empirical variance and this enables us to prove Lemma~\ref{l:gauss-regret}, which holds for both the CH-AS and the B-AS algorithm.

We now report a regret bound in the case of the Gaussian distribution. Note that in this case Assumption~\ref{a:subgaussian} holds with $c_1 = 2\Sigma$ and $c_2 = 1$.\footnote{Note that for a single Gaussian distribution $c_1 = 2\sigma^2$, where $\sigma^2$ is the variance of the distribution. Here we use $c_1 = 2\Sigma$ in order for the assumption to be satisfied for all the $K$ distributions simultaneously.}

\begin{theorem}\label{thm:b-regret-gauss}
Assume that all the distributions $\{\nu_k\}_{k=1}^K$ are Gaussian and that an upper-bound $\overline{\Sigma}\geq 1/2$ on $\Sigma$ is known. If the budget is known on advance and such that $n\geq 5K$, the B-AS algorithm launched with parameters $c_1=2\overline{\Sigma}$, $c_2=1$, and $\de=n^{-7/2}$ has the following regret bound
\begin{align}\label{eq:b-regret}
R_n(\alg_B)  &\leq  
 \frac{105\times 10^3 \bar \Sigma}{n^{3/2}}K^2(\log n)^2\;.
\end{align}
\end{theorem}

\begin{proof}
The proof is reported in~\ref{s:b-results-gauss}.
\end{proof}

\paragraph{Remark 1} In the case of Gaussian distributions, the regret bound for B-AS has the rate $\tilde O(n^{-3/2})$ without dependency on $\lambda_{\min}$, which represents a significant improvement over the regret bounds of the CH-AS and GAFS-MAX algorithms. 

\paragraph{Remark 2} In practice, there is no need to tune the three parameters $c_1$, $c_2$, and $\de$ separately. In fact, it is enough to tune the algorithm for a single parameter $a\sqrt{\log(2/\de)}$ (see Figure~\ref{f:b-algorithm}). Using the proof of Theorem~\ref{thm:b-regret} and the optimized value of $\de$, as well as the fact that for Gaussian distributions, $c_1 \leq 2\Sigma$, and $c_2 \leq 1$, it is possible to show that choosing $a$ as in Theorem~\ref{thm:b-regret-gauss} means that $a = O\big((\overline{\Sigma}\log n)^{1/2}\big)$, where $\overline{\Sigma}$ is an upper bound on the value of $\Sigma$. This is a reasonable thing to do whenever a rough estimate of the magnitude of the variances is available.



\section{Experimental Results}\label{s:experiments}


\subsection{CH-AS, B-AS, and GAFS-MAX with Gaussian Arms}\label{ss:CH-B-G}

In this section, we compare the performance of CH-AS, B-AS, and GAFS-MAX on a two-armed problem with Gaussian distributions $\distro_1 = \mathcal N(0, \si_1^2=4)$ and $\distro_2 = \mathcal N(0, \si_2^2=1)$ (note that $\lambda_{\min} \!\!=\!\! 1/5$). Figure~\ref{f:comparison}-{\em (left)} shows the rescaled regret, $n^{3/2} R_n$, for the three algorithms averaged over $50,000$ runs. The results indicate that while the rescaled regret is almost constant with respect to~$n$ in B-AS and GAFS-MAX, it increases for small (relative to $\lambda^{-1}_{\min}$) values of $n$ in CH-AS. 

The robust behavior of B-AS when the distributions of the arms are Gaussian may be easily explained by the bound of Theorem~\ref{thm:b-regret-gauss} (Equation~\ref{eq:b-regret}). Note though that this experiment seems to imply that there is no additional dependency in $\log(n)$: it could be just an artifact of the proof. The initial increase in the CH-AS curve is also consistent with the bound of Theorem~\ref{thm:ch-regret} (Equation~\ref{eq:ch-regret}). As discussed in Remark~3 of Section~\ref{ss:ch-discussion}, we conjecture that the regret bound for CH-AS is of the form \begin{small}$R_n \leq \min\big\lbrace \lambda_{\min}^{-5/2} \tilde O(n^{-3/2}), \tilde O(n^{-4/3})\big\rbrace$\end{small}, and thus, the algorithm's regret is bounded as $\tilde O(n^{-4/3})$ and $\lambda_{\min}^{-5/2}\tilde O(n^{-3/2})$ for small and large (relative to $\lambda^{-1}_{\min}$) values of $n$, respectively. It is important to note that the regret bound of CH-AS depends on the arms' distributions only through the variances of the distributions, as shown in Theorem~\ref{l:ch-upper-lower}. Finally, the curve for GAFS-MAX is very close to the curve for B-AS. For this reason, we believe that it could be possible to improve the GAFS-MAX analysis by using refined concentration inequalities for the standard deviation as done in B-AS. This might also remove the inverse dependency on $\lambda_{\min}$ and provide a regret bound similar to B-AS in the case of Gaussian distributions.

\begin{figure}[htp]
\begin{center}
\includegraphics[width=0.4\columnwidth]{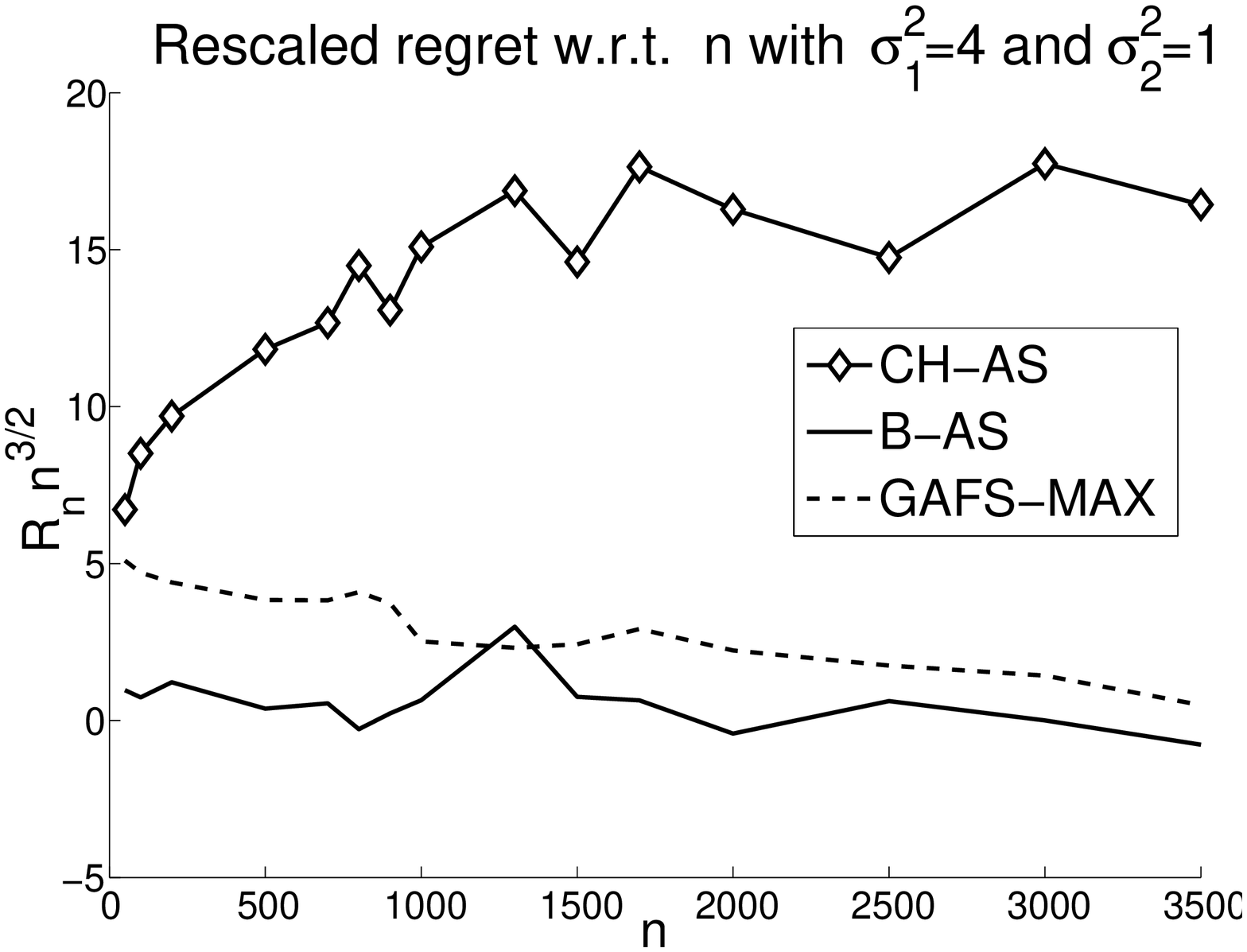}\hspace{0.75in}
\includegraphics[width=0.4\columnwidth]{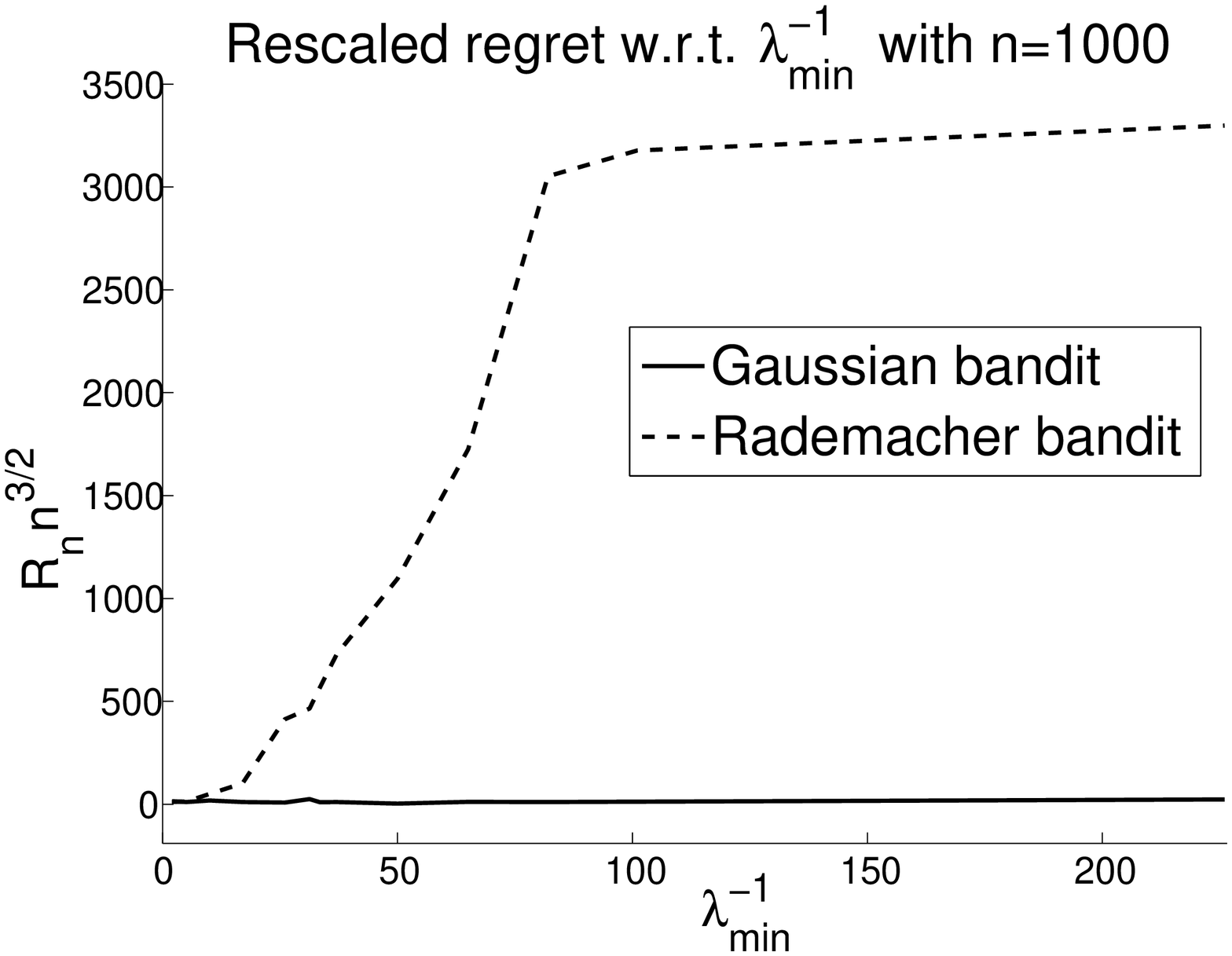}
\caption{{\em (left)} The rescaled regret of CH-AS, B-AS, and GAFS-MAX algorithms on a two-armed problem, where the distributions of the arms are Gaussian. {\em (right)} The rescaled regret of B-AS for two bandit problems, one with two Gaussian arms and one with a Gaussian and a Rademacher arms.}
\label{f:comparison}
\end{center}
\end{figure}


\subsection{B-AS with Non-Gaussian Arms}\label{ss:non-gauss}

In Section~\ref{sec:gaussian.dist}, we showed that when the arms have Gaussian distribution, the regret bound of the B-AS algorithm no longer depends on $\lambda_{\min}$. We also discussed why we conjecture that it is not possible to remove this dependency for general distributions unless a tighter upper bound on the number of pulls can be derived. Although we do not yet have a lower bound on the regret showing the dependency on $\lambda_{\min}$, i.e.~that the regret might depend on the shape of the distribution, in this section we show that for Rademacher distributions, the regret of B-AS behaves in a different way than for Gaussian distributions with same variance.

As discussed in Section~\ref{sec:gaussian.dist}, the property of the Gaussian distribution that allows us to remove the $\lambda_{\min}$ dependency in the regret bound of B-AS is that for any sample of fixed size drawn i.i.d. from a Gaussian distribution, the corresponding empirical mean and the empirical variance are independent. The quantities $(\hmu_{k,n} -\mu_k)^2$ and $\hsi_{k,n}$ are however conditionally negatively correlated given $T_{k,n}$ for e.g.,~the Rademacher distribution.\footnote{$X$ is Rademacher if $X\in\{-1,1\}$ and admits values $-1$ and $1$ with equal probability.} In the case of Rademacher distribution, the loss $(\hmu_{k,t} - \mu_k)^2$ is equal to $\hmu_{k,t}^2$ and we have $\hvar_{k,t}=\frac{1}{T_{k,t} -1}\Big(\sum_{i=1}^{T_{k,t}} X_{k,i}^2 - T_{k,t} \hmu_{k,t}^2\Big) = \frac{T_{k,t}}{T_{k,t} -1} \Big(1-\hmu_{k,t}^2\Big)$, as a result, the larger $\hvar_{k,t}$ is, the smaller $\hmu_{k,t}^2$ is. We know that the allocation strategies in CH-AS, B-AS, and GAFS-MAX are based on the empirical variance which is used as a substitute for the true variance. As a result, the larger $\hvar_{k,t}$ is, the more often arm $k$ is pulled. For the Rademacher distribution, this means that an arm is pulled more than its optimal allocation when its mean is accurately estimated (the loss is small). This may result in a poor estimation of the arm, and thus, negatively affect the regret of the algorithm. 

In the experiments of this section, we use B-AS in two different bandit problems: one with two Gaussian arms $\distro_1 = \mathcal N(0,\sigma_1^2)$ (with $\sigma_1\geq 1$) and $\distro_2 = \mathcal N(0,1)$, and one with a Gaussian $\distro_1 = \mathcal N(0,\sigma_1^2)$ (with $\sigma_1\geq 1$) and a Rademacher $\distro_2$ arms. Note that in both cases $\lambda_{\min} = \lambda_2 = 1/(1+\sigma_1^2)$. Figure~\ref{f:comparison}-{\em (right)} shows the rescaled regret ($n^{3/2}R_n$) of the B-AS algorithm as a function of $\lambda_{\min}^{-1}$ for $n=1000$. While the rescaled regret of B-AS is constant in the first problem, it increases with $\sigma_1^2$ in the second one. This leads us to the conclusion that the shape of the distributions of the arms has an impact on the regret of the algorithm B-AS. In fact, as explained above, this behavior might be due to the poor approximation of the Rademacher arm which is over-pulled exactly whenever its estimated mean is accurate. This result seems to illustrates the fact that in this active learning problem (where the goal is to estimate the mean values of the arms), the performance of the algorithms that rely on the empirical-variance (e.g.,~CH-AS, B-AS, and GAFS-MAX) depends on the shape of the distributions, and not only on their variances. This may be surprising since according to the central limit theorem the distribution of the empirical mean should tend to a Gaussian. However, it seems that what is important is not the distribution of the empirical mean or variance, but the correlation of these two quantities. This is why we believe that any algorithm that is based on empirical standard deviations might be subject to the same problem. However, at the moment no full satisfactory theoretical analysis is available on this point.

\section{Conclusions and Open Questions}\label{s:conclusions}

In this paper, we studied the problem of adaptive allocation for finding a uniformly good estimation of the mean values of $K$ independent distributions. This problem was first studied by~\citet{antos2010active}. Although the algorithm proposed in~\cite{antos2010active} achieves a small regret of order $\tilde O(n^{-3/2})$, it displays an inverse dependency on the smallest proportion $\lambda_{\min}$. In this paper, we first introduced a novel class of algorithms based on upper-confidence-bounds on the (unknown) variances of the arms, and analyzed two such algorithms: Chernoff-Hoeffding allocation strategy (CH-AS) and Bernstein allocation strategy (B-AS). For CH-AS we derived a regret similar to~\cite{antos2010active}, scaling as $\tilde O(n^{-3/2})$ and with the dependence on $\lambda_{\min}$. Unlike in~\cite{antos2010active}, this result holds for any $n\geq 5K$ and the constants in the bound are made explicit. We then introduced a more refined algorithm, B-AS, whose regret bound does not depend on $\lambda_{\min}$ for Gaussian arms. Nonetheless, its general regret bound still depends on $\lambda_{\min}$. We show that this dependency may be related to the specific distributions of the arms and can be removed for the case of Gaussian distributions. Finally, we report numerical simulations supporting the idea that the shape of the distributions has an impact on the performance of the allocation strategies. \\

\noindent
This work opens a number of questions.
\begin{itemize}
\item \textit{Distribution dependency.} Another open question is to which extent the result of B-AS in the case of the Gaussian distribution can be extended to more general families of distributions. As illustrated in the case of Rademacher, the correlation between the empirical mean and variance may cause the algorithm to over-pull arms even when their estimation is accurate, thus incurring a large regret. On the other hand, if the distributions of the arms are Gaussian, their empirical mean and variance are uncorrelated and the allocation algorithms such as B-AS achieve a better regret. Further investigation is needed to identify whether this result can be extended to other distributions.
\item \textit{Lower bound.} The results of Sections~\ref{sec:gaussian.dist} and~\ref{ss:non-gauss} suggest that the dependency on the distributions of the arms could be intrinsic to the allocation problem. If this is the case, it should be possible to derive a lower bound for this problem showing such dependency (a lower-bound with dependency on $\lambda_{\min}^{-1}$). As a matter of fact, no lower bounds are available for this problem and it would be interesting to provide some.
\end{itemize}

\subsection*{Acknowledgment} This work was supported by French National Research Agency (ANR) through the projects EXPLO-RA $n^\circ$ ANR-08-COSI-004 and LAMPADA $n^\circ$ ANR-09-EMER-007, by Ministry of Higher Education and Research, Nord-Pas de Calais Regional Council and FEDER through the ``contrat de projets {\'e}tat region (CPER) 2007--2013", European Community's Seventh Framework Programme (FP7/2007-2013) under grant agreement $n^\circ$ 231495, and by PASCAL2 European Network of Excellence.


\bibliography{allocation}


\newpage
\appendix

\section{Regret Bound for the CH-AS Algorithm}\label{ss:ch-proof}


Let us consider $n>0$ and $\de>0$ (that can be a function of $n$) fixed. We consider all the quantities considered in the definition of algorithm CH-AS defined with respect to these fixed $n$, $\de$, and use the abbreviated notations $\hmu_{k,t}$, $\hsi_{k,t}$, $B_{k,t}$, $k_t$, and $T_{k,t}$. 

\subsection{Basic Tools}

Since the basic tools used in the proof of Theorem~\ref{thm:ch-regret} are similar to those used in the work by~\citet{antos2010active}, we begin this section by restating two results from that paper. Let $\xi$ be the event
\begin{equation}\label{eq:ch-var1}
\xi = \xi_{K,n}^{CH}(\de) = \mathop{\bigcap_{1\leq k\leq K}}_{1\leq t\leq n}\left\lbrace \Big|\Big(\frac{1}{t}\sum_{i=1}^tX_{k,i}^2-\big(\frac{1}{t}\sum_{i=1}^tX_{k,i}\big)^2\Big) - \si_k^2\Big| \leq 3 \sqrt{\frac{\log\invdelta}{2t}} \right\rbrace.
\end{equation}
%
Note that the first term in the absolute value in Equation~\eqref{eq:ch-var1} is the sample variance of arm $k$ computed as in Equation~\eqref{eq:estim-var1} for $t$ samples. It can be shown using Hoeffding's inequality (see~\citet{Hoeffding63PI}) that $\Pr[\xi]\geq1-4nK\de$, and this is shown by directly reusing the elements of the proof of Lemma~2 in~\citet{antos2010active}. The event $\xi$ plays an important role in the proofs of this section and several statements will be proved on this event. We now report the following proposition which is analog to Lemma~2 in~\citet{antos2010active}.

\begin{proposition}\label{p:ch-var}
For any $k=1,\ldots,K$ and $t=1,\ldots,n$, let $\{X_{k,i}\}_{i=1,\ldots,T_{k,t}}$ be $T_{k,t} \in \{1,\ldots,t\}$ i.i.d.~random variables bounded in $[0,1]$ from the distribution $\distro_k$ with variance $\var_k$, and $\hvar_{k,t}$ be the sample variance computed as in Equation~\eqref{eq:estim-var1}. Then the following statement holds on the event $\xi$:
\begin{equation}\label{eq:ch-var2}
|\hsi_{k,t}^2 - \si_k^2| \leq 3 \sqrt{\frac{\log\invdelta}{2T_{k,t}}}\;.
\end{equation}
\end{proposition}

%
%
%
%
%
%
%
%

We also need to draw a connection between the allocation and stopping time problems. Thus, we report the following proposition which is Lemma~10 in~\citet{antos2010active}. 

\begin{proposition}\label{p:wald-inequality}
Let $\{\mathcal F_t\}_{t=1,\ldots,n}$ be a filtration and $\{X_t\}_{t=1,\ldots,n}$ be an $\mathcal F_t$ adapted sequence of i.i.d.~random variables with finite expectation $\mu$ and variance $\si^2$. Assume that $\mathcal F_t$ and $\sigma(\{X_s:s \geq t+1\})$ are independent for any $t \leq n$, and let $T(\leq n)$ be a stopping time with respect to $\mathcal F_t$. Then
%
\begin{equation}\label{eq:wald-inequality}
\E\Bigg[\Big(\sum_{i=1}^T X_i - T\;\mu\Big)^2\Bigg] = \E[T]\;\var.
\end{equation}
\end{proposition}


\subsection{Allocation Performance}\label{a:l:ch-upper-lower}

In this subsection, we first provide the proof of Lemma~\ref{l:ch-upper-lower} and then use the result in the next subsection to prove Theorem~\ref{thm:ch-regret}. \\

%
%




\begin{proof}[Proof of Lemma~\ref{l:ch-upper-lower}]
The proof consists of the following three main steps. We assume that $\xi$ holds until the end of this proof. \\

\noindent
\textbf{Step~1.~Mechanism of the algorithm.} Recall the definition of the upper bound used in $\alg_{CH}$ at a time $t+1 > 2K$:

\begin{equation*}
B_{q,t+1} = \frac{1}{T_{q,t}} \Bigg(\hsi_{q,t}^2 + 3\sqrt{\frac{\log\invdelta}{2T_{q,t}}}\Bigg),\quad\quad 1\leq q\leq K\;.
\end{equation*}

\noindent From Proposition~\ref{p:ch-var}, we obtain the following upper and lower bounds for $B_{q,t+1}$ on the event $\xi$:

\begin{equation}\label{e:bound.B2}
 \frac{\var_q}{T_{q,t}} \leq B_{q,t+1} \leq \frac{1}{T_{q,t}} \Bigg(\var_{q} + 6\sqrt{\frac{\log\invdelta}{2T_{q,t}}}\Bigg).
\end{equation}

\noindent Note that as $n\ge 4K$, there is at least one arm $k$ that is pulled after the initialization. Let $k$ be a given such arm and $t+1>2K$ be the time when it is pulled for the last time, i.e., $T_{k,t}=T_{k,n}-1$ and $T_{k,t+1}=T_{k,n}$. Since $\alg_{CH}$ chooses to pull arm $k$ at time $t+1$, for any arm $p$, we have

\begin{equation}\label{e:meca1}
B_{p,t+1} \leq B_{k,t+1}\;.
\end{equation}

\noindent From Equation~\eqref{e:bound.B2} and the fact that $T_{k,t} = T_{k,n} -1$, we obtain

\begin{equation}\label{e:meca.lb1}
B_{k,t+1} \leq \frac{1}{T_{k,t}} \Bigg(\var_{k} + 6\sqrt{\frac{\log\invdelta}{2T_{k,t}}}\Bigg) = \frac{1}{T_{k,n}-1} \Bigg(\var_{k} + 6\sqrt{\frac{\log\invdelta}{2(T_{k,n}-1)}}\Bigg).
\end{equation}

\noindent Using the lower bound in Equation~\eqref{e:bound.B2} and the fact that $T_{p,t} \leq T_{p,n}$, we may lower bound $B_{p,t+1}$ as

\begin{equation}\label{e:meca.hb1}
B_{p,t+1} \geq \frac{\var_{p}}{T_{p,t}} \geq \frac{\var_{p}}{T_{p,n}}\;.
\end{equation}

\noindent Combining Equations~\ref{e:meca1},~\ref{e:meca.lb1}, and~\ref{e:meca.hb1}, we obtain

\begin{equation}\label{e:bound.fp}
\frac{\var_{p}}{T_{p,n}} \leq \frac{1}{T_{k,n}-1} \Bigg(\var_{k} + 6\sqrt{\frac{\log\invdelta}{2(T_{k,n}-1)}}\Bigg)\;.
\end{equation}

\noindent Note that at this point there is no dependency on $t$, and thus, Equation~\eqref{e:bound.fp} holds on the event $\xi$ for any arm $k$ that is pulled at least once after the initialization, and for any arm $p$. \\
%

\noindent
{\bf Step~2.~Lower bound on $T_{p,n}$.} If an arm $q$ is under-pulled \textit{without taking into account the initialization phase}, i.e.,~$T_{q,n}-2 < \lambda_q(n-2K)$, then from the constraint $\sum_k (T_{k,n}-2) = n-2K$, we deduce that there must be at least one arm $k$ that is over-pulled, i.e.,~$T_{k,n}-2 > \lambda_k(n-2K)$. Note that for this arm, $T_{k,n}-2 > \lambda_k (n-2K)\geq0$, so we know that this specific arm is pulled at least once \textit{after} the initialization phase and that it satisfies Equation~\eqref{e:bound.fp}. Using the definition of the optimal (up to rounding effects) allocation $T_{k,n}^* = n\lambda_k = n\sigma^2_k /\Sigma$ and the fact that $T_{k,n} \geq \lambda_k(n-2K) + 2$, Equation~\eqref{e:bound.fp} may be written as

\begin{align}\label{eq:almost.lb}
\frac{\var_{p}}{T_{p,n}} &\leq \frac{1}{T_{k,n}^*}\frac{n}{n-2K} \Bigg(\var_{k} + 6\sqrt{\frac{\log\invdelta}{2(\lambda_k(n-2K)+2-1)}}\Bigg) \nonumber \\
&\leq \frac{\Sigma}{n-2K} + \frac{12\sqrt{\log\invdelta}}{(\lambda_{\min}n)^{3/2}} \nonumber \\
&\leq \frac{\Sigma}{n} + \frac{12\sqrt{\log\invdelta}}{(\lambda_{\min}n)^{3/2}}+\frac{4K\Sigma}{n^2},
\end{align}

\noindent
since $\lambda_{k}(n-2K)+1 \geq \lambda_{k}(n/2 - 2K + 2K)+1 \geq \frac{n\lambda_{k}}{2}$, as $n\geq 5K$ (thus also $\frac{2K\Sigma}{n(n-2K)} \leq \frac{4K\Sigma}{n^2}$). Also, if no arm is under-pulled after time $2K$, then for each $p$, $T_{p,n} \geq 2 +\lambda_p(n-2K) > \lambda_p(n-2K)$, i.e., $\sigma_p^2/T_{p,n} \leq \sigma_p^2/(\lambda_p(n-2K)) = \Sigma/(n-2K)$, i.e., Equation~\eqref{eq:almost.lb} holds anyway (whether there are under-pulled arms or not).
%
%
%
 By reordering the terms in the previous equation, we obtain the lower bound

\begin{align}\label{eq:ch-lower-bound-p}
T_{p,n} \geq\frac{\sigma_p^2}{\frac{\Sigma}{n} + \frac{12\sqrt{\log\invdelta}}{(n \lambda_{\min})^{3/2}} +\frac{4K\Sigma}{n^2}} \geq T_{p,n}^* - \lambda_p \frac{12}{\Sigma\lambda_{\min}^{3/2}} \sqrt{n\log\invdelta} - 4\lambda_p K,
\end{align}

\noindent
where in the second inequality we used $1/(1+x) \geq 1-x$ (for $x> -1$). Note that the lower bound~\ref{eq:ch-lower-bound-p} holds on $\xi$ for any arm $p$. \\

\noindent
{\bf Step~3.~Upper bound on $T_{p,n}$.} Using Equation~\eqref{eq:ch-lower-bound-p} and the fact that $\sum_k T_{k,n} = \sum_k T_{k,n}^* = n$, we obtain the upper bound

\begin{align}\label{eq:ch-upper-bound-p}
T_{p,n} = n - \sum_{k\neq p}T_{k,n} \leq T_{p,n}^* + \frac{12}{\Sigma\lambda_{\min}^{3/2}} \sqrt{n\log\invdelta} + 4K\;.
\end{align}

\noindent The claim follows by combining the lower and upper bounds in Equations~\ref{eq:ch-lower-bound-p} and~\ref{eq:ch-upper-bound-p}.
\end{proof}


\subsection{Regret Bound}\label{a:thm:ch-regret}

We now show how the bound on the allocation over arms translates into a bound on the regret of the algorithm as stated in Theorem~\ref{thm:ch-regret}.\\
%

\begin{proof}[Proof of Theorem \ref{thm:ch-regret}]
The proof consists of the following two main steps. \\

\noindent
{\bf Step~1. For each $1\le n'\le n$, $T_{k,n'}$ is a stopping time.} 
For a given $k$, let $(\F^{(k)}_t)_{t\le n}$ be the filtration associated to the process $\{X_{k,t}\}_{t\le n}$, and $\mathcal{E}_{-k}=\mathcal{E}_{-k,n}$ be the $\sigma$-algebra generated by $\{X_{k',t'}\}_{t'\le n,k'\ne k}$ (``environment''). Let $\G^{(k)}_t = \G^{(k,n)}_t = \sigma(\F^{(k)}_t,\mathcal{E}_{-k})$.

We prove for fixed budget $n$ by induction for $n'=1,\ldots,n$ that each $T_{k,n'}$ is a stopping time with respect to the filtration $(\G^{(k)}_t)_{t\le n}$.

For $n'\le 2K$ (initialization), $T_{k,n'}$ is deterministic, so for any $t$, $\{T_{k,n'}\le t\}$ is either the empty set or the whole probability space (and is thus measurable according to $\G^{(k)}_t$).

Let us now assume that for a given time step $2K\le n'<n$, and for any $t$, $\{T_{k,n'}\le t\}$ is $\G^{(k)}_t$-measurable.
We consider now time step $n'+1$. Note first that for $t=0$, $\{T_{k,n'+1} \le t\}=\{T_{k,n'+1} \le 0\}$ is the empty set and is thus $\G^{(k)}_t$-measurable.
If $t>0$, then
\begin{equation}\label{eq:recursion}
 \{T_{k,n'+1}\le t\}
 = \left(\{T_{k,n'} = t\}\cap \{k_{n'+1} \ne k\}\right)
 \cup \{T_{k,n'}\le t-1\}.
\end{equation}
By induction assumption, $\{T_{k,n'} = t\}$ and $\{T_{k,n'} \le t-1\}$ are $\G^{(k)}_t$-measurable (since \emph{for any} $t'$, $\{T_{k,n'}\le t'\}$ is $\G^{(k)}_{t'}$-measurable). On $\{T_{k,n'} = t\}$, $k_{n'+1}$ is also $\G^{(k)}_t$-measurable since it is determined only by the values of the upper-bounds $\{B_{q,n'+1}\}_{1\le q\le K}$ (which depend only on $\{X_{k',t'}\}_{t'\le n,k'\ne k}$ and on $(X_{k,1},\ldots,X_{k,t})$). Hence, $\{T_{k,n'} = t\}\cap\{k_{n'+1} \ne k\}$ is $\G^{(k)}_t$-measurable, and thus using (\ref{eq:recursion}), we have that $\{T_{k,n'+1}\le t\}$ is $\G^{(k)}_t$-measurable, as well.

We have thus proved by induction that $T_{k,n'}$ is a stopping time with respect to the filtration $(\G^{(k)}_t)_{t\le n}$.

\noindent
{\bf Step~2.~Regret bound.} Using its definition, we may write $L_{k,n}$ as follow:

\begin{equation*}
L_{k,n} = \E\Big[(\hmu_{k,n} - \mu_k)^2 \Big] = \E\Big[(\hmu_{k,n} - \mu_k)^2 \1\{{\xi\}}\Big] + \E\Big[(\hmu_{k,n} - \mu_k)^2 \1\{{\xi^{C}\}}\Big].
\end{equation*}

\noindent Using the definition of $\hmu_{k,n}$ and Proposition~\ref{p:wald-inequality} for filtration $\{\G^{(k)}_t\}_{t\leq n}$, $\{X_{k,t}\}_{t\leq n}$,
and $T_{k,n}$ (and that $\G^{(k)}_t = \sigma(\{X_{k,t'}:t'\le t\}\cup\{X_{k',t'}:t'\leq n,k'\neq k\})$ and $\sigma(\{X_{k,t'}:t'\geq t+1\})$ are independent for any $t\leq n$) we bound the first term as
\begin{align}
\E\Big[(\hmu_{k,n} - \mu_k)^2 \1\{{\xi\}}\Big] &\leq  \sup_{\omega \in \xi}\Big(\frac{\si_k^2}{T_{k,n}^2(\omega)} \Big)\E\Big[\frac{(\sum_{t=1}^{T_{k,n}} X_{k,t} - T_{k,n}\mu_k)^2}{\si_k^2} \1\{{\xi\}}\Big] \nonumber \\
&\leq  \sup_{\xi}\Big(\frac{\si_k^2}{T_{k,n}^2} \Big) \E\Big[\frac{1}{\si_k^2}(\sum_{t=1}^{T_{k,n}} X_{k,t} - T_{k,n}\mu_k)^2 \Big] \nonumber\\
&=  \sup_{\xi}\Big(\frac{\si_k^2}{T_{k,n}^2} \Big) \frac{1}{\si_k^2} \si_k^2 \E[T_{k,n}] \nonumber\\
 &= \sup_{\xi}\Big(\frac{\si_k^2}{T_{k,n}^2} \Big) \E[T_{k,n}]\;, \label{eq:st2.1}
\end{align}

Since the upper-bound in Lemma~\ref{l:ch-upper-lower} is obtained on the event $\xi$ (and thus with high probability), and as $T_{k,n} \leq n$, we may easily convert it to a bound in expectation as follows:

\begin{equation}\label{e:bound.lbtpnu.esp}
\E[T_{k,n}] \leq \Big(T_{k,n}^* + \frac{12}{\Sigma\lambda_{\min}^{3/2}} \sqrt{n\log\invdelta} + 4K\Big) + n \times 4nK\de.
\end{equation}

\noindent Combining Equation~\eqref{eq:st2.1} and~\ref{e:bound.lbtpnu.esp}, and using Equation~\eqref{eq:almost.lb} for $\sup_{\xi} \Big(\var_{k}/T_{k,n}\Big)$, we obtain 

\begin{align}\label{e:bound.reg.stime1}
&\E\Big[(\hmu_{k,n} - \mu_k)^2 \1\{{\xi\}}\Big] \nonumber\\ 
&\leq \Bigg(\frac{\Sigma}{n} + \frac{12\sqrt{\log\invdelta}}{(\lambda_{\min}n)^{3/2}}+\frac{4K\Sigma}{n^2}\Bigg)^2\frac{\Big(T_{k,n}^* + \frac{12}{\Sigma\lambda_{\min}^{3/2}} \sqrt{n\log\invdelta} + 4K + n \times 4nK\de\Big)}{\var_k}.
\end{align}


\noindent By setting $A=\frac{12\sqrt{\log\invdelta}}{\lambda_{\min}^{3/2}}$ to simplify the notation, Equation~\eqref{e:bound.reg.stime1} may be simplified as 

\begin{align*}
&\E\Big[(\hmu_{k,n} - \mu_k)^2 \1\{{\xi\}}\Big]\\ 
&\leq \Bigg(\frac{\Sigma}{n} + \frac{A}{n^{3/2}}+\frac{4K\Sigma}{n^2}\Bigg)^2 \Bigg(\frac{n}{\Sigma} + \frac{A}{\Sigma\si_k^2}\sqrt{n} + \frac{4K + 4n^2K\de}{\si_k^2}\Bigg)\\
&= \Bigg(\frac{\Sigma^2}{n^2} + \frac{A^2}{n^{3}}+\frac{16K^2\Sigma^2}{n^4}+\frac{2A\Sigma}{n^{5/2}}+\frac{8K\Sigma^2}{n^3}+\frac{8AK\Sigma}{n^{7/2}}\Bigg) \Big(\cdots\Big)\\
&= \Bigg(\frac{\Sigma^2}{n^2} + \frac{2A\Sigma}{n^{5/2}}+ \frac{1}{n^3}\Big(A^2+\frac{16K^2\Sigma^2}{n}+8K\Sigma^2+\frac{8AK\Sigma}{n^{1/2}}\Big)\Bigg) \Big(\cdots\Big),\\
&\leq \Bigg(\frac{\Sigma^2}{n^2} + \frac{2A\Sigma}{n^{5/2}}+ \frac{1}{n^3}\Big(A^2+12K\Sigma^2+4A\sqrt{K}\Sigma\Big)\Bigg) \Big(\cdots\Big),
 \end{align*}

\noindent where in the last passage we used $n \geq 5K$. Let $B = A^2+12K\Sigma^2+4A\sqrt{K}\Sigma$. We further simplify the previous expression as

\begin{align*}
&\E\Big[(\hmu_{k,n} - \mu_k)^2 \1\{{\xi\}}\Big]\\ 
&\leq \frac{\Sigma}{n} + \frac{1}{n^{3/2}}\Big( \frac{\Sigma A}{\sigma^2_k} + 2A\Big) + \frac{1}{n^2}\Big( \frac{4K\Sigma^2}{\si_k^2} + \frac{2A^2}{\sigma^2_k} + \frac{B}{\Sigma}\Big) + \frac{1}{n^{5/2}}\Big( \frac{8\Sigma AK}{\si_k^2} + \frac{AB}{\sigma^2_k\Sigma} \Big) + \frac{4KB}{\si_k^2 n^3} \\
&+ \Big(\frac{4  K\Sigma^2}{\si_k^2} + \frac{8\Sigma AK}{\si_k^2 n^{1/2}} + \frac{4KB}{\si_k^2 n}\Big)\de.
 \end{align*}

\noindent We now choose $\de =  n^{-5/2}$ and by using $n \geq 5K$ and $\lambda_{\min} \leq 1/K$ we obtain
\begin{small}
\begin{align*}
&\E\Big[(\hmu_{k,n} - \mu_k)^2 \1\{{\xi\}}\Big]\\ 
&\leq \frac{\Sigma}{n} + \frac{1}{n^{3/2}}\Big( \frac{\Sigma A}{\sigma^2_k} + 2A\Big) + \frac{1}{n^2}\Big(\frac{4K\Sigma^2}{\si_k^2} + \frac{2A^2}{\si_k^2} + \frac{B}{\Sigma} + \frac{4\Sigma A\sqrt{K}}{\si_k^2} + \frac{AB}{2\sqrt{K}\si_k^2 \Sigma} + \frac{B}{\si_k^2} + \frac{2\Sigma^2 \sqrt{K}}{\si_k^2} + \frac{2\Sigma A}{\si_k^2} + \frac{B}{2 \sqrt{K}\si_k^2}\Big)\\
&\leq \frac{\Sigma}{n} + \frac{1}{n^{3/2}}\Big( \frac{\Sigma A}{\sigma^2_k} + 2A\Big) + \frac{1}{\lambda_{\min}n^2}\Big(4K\Sigma + \frac{2A^2}{\Sigma}+ \frac{B}{K\Sigma} + 4A\sqrt{K} + \frac{AB}{2\Sigma^2 \sqrt{K}} + \frac{B}{\Sigma} + 
2\Sigma\sqrt{K} + 2A + \frac{B}{2\sqrt{K}\Sigma}\Big)\\
&= \frac{\Sigma}{n} + \frac{1}{n^{3/2}}\Big( \frac{\Sigma A}{\sigma^2_k} + 2A\Big) + \frac{1}{\lambda_{\min}n^2}\Big(4K\Sigma + 2\Sigma\sqrt{K}  + 4A\sqrt{K}+ 2A + \frac{2A^2}{\Sigma}+  \frac{B}{\Sigma}  + \frac{B}{2\sqrt{K}\Sigma}+ \frac{B}{K\Sigma}  + \frac{AB}{2\Sigma^2 \sqrt{K}} \Big)\\
&\leq \frac{\Sigma}{n} + \frac{1}{n^{3/2}}\Big( \frac{\Sigma A}{\sigma^2_k} + 2A\Big) + \frac{1}{\lambda_{\min}n^2}\Big(1.4K^2 + A(4\sqrt{K}+ 2) + \frac{2A^2}{\Sigma}+  \frac{B}{\Sigma}  + \frac{B}{4\Sigma^{3/2}}+ \frac{B}{K\Sigma}  + \frac{AB}{4\Sigma^{5/2}} \Big),
 \end{align*}
\end{small}
\noindent where the last passage follows from $\Sigma \leq K/4$. 
 
\noindent Before proceeding further we notice that $\lambda_{\min} \leq 1/K$ and thus
\begin{align*}
K^{3/2} \leq \frac{1}{\lambda_{\min}^{3/2}} = \frac{A}{12\sqrt{\log(1/\de)}} \leq \frac{A}{12\sqrt{(5/2) \log n}} \leq \frac{A}{27},
\end{align*}
where the first passage follows from the definition of $A$ and the second from $\de = n^{-5/2}$, and $n \geq 5K \geq 10$. This implies by definition of $B$
\begin{align*}
B = A^2+12K\Sigma^2+ 4A\sqrt{K}\Sigma \leq A^2 + 3A^2/27^2/4 + A^2/27 = 1009A^2/972 < 27A^2/26 < 1.05A^2,
 \end{align*}
where we use $\Sigma \leq K/4$. By using the previous bound, we finally obtain since $ 1.4K^2 \le 0.7K^3 \le 0.7A^2/27^2 \le A^2/1041$
\begin{align*}
&\E\Big[(\hmu_{k,n} - \mu_k)^2 \1\{{\xi\}}\Big]\\ 
&\leq \frac{\Sigma}{n} + \frac{1}{n^{3/2}}\Big( \frac{\Sigma A}{\sigma^2_k} + 2A\Big) + \frac{1}{\lambda_{\min}n^2}\Big( 1.4K^2  + A(4\sqrt{K}+ 2) + \frac{2A^2}{\Sigma}+  \frac{1.05 A^2}{\Sigma}  + \frac{1.05 A^2}{4\Sigma^{3/2}}+ \frac{1.05 A^2}{K\Sigma}  + \frac{1.05 A^3}{4\Sigma^{5/2}}\Big)\\
&\leq \frac{\Sigma}{n} + \frac{1}{n^{3/2}}\Big( \frac{\Sigma A}{\sigma^2_k} + 2A\Big) + \frac{1}{\lambda_{\min}n^2}\Big(A^2/1041+ 0.9 A^{3/2} + 3.6\big(\frac{1}{\Sigma}+\frac{1}{\Sigma^2}\big)A^2 + \frac{1.05A^3}{4\Sigma^{5/2}} \Big)\\
&\leq \frac{\Sigma}{n} + \frac{1}{n^{3/2}} \frac{2 A}{\lambda_{\min}} + \frac{1}{\lambda_{\min}n^2}\Big(0.9 A^{3/2} +3.7\big(\frac{1}{\Sigma}+\frac{1}{\Sigma^2}\big)A^2+ \frac{0.27A^3}{\Sigma^{5/2}} \Big).
 \end{align*}
 
 \noindent Since $|\hmu_{k,n} - \mu_k|$ is always smaller than $1$, we have $\E\big[(\hmu_{k,n} - \mu_k)^2 \1\{{\xi^{C}\}}\big] \leq 4nK\de = 4Kn^{-3/2}$. We also know that $A\leq \frac{19\sqrt{\log(n)}}{\lambda_{\min}^{3/2}}$. Thus the expected loss of arm $k$ is bounded by

\begin{align*}\label{e:bound.reg.stime4}
L_{k,n} &\leq \frac{\Sigma}{n} + \frac{38\sqrt{\log(n)}}{n^{3/2} \lambda_{\min}^{5/2}} + \frac{1}{\lambda_{\min}n^2}\Big(0.9 A^{3/2} +3.7\big(\frac{1}{\Sigma}+\frac{1}{\Sigma^2}\big)A^2+ \frac{0.27A^3}{\Sigma^{5/2}}\Big)+ 4nK\de\\
&\leq \frac{\Sigma}{n} +\frac{39\sqrt{\log(n)}}{n^{3/2} \lambda_{\min}^{5/2}} +  \frac{2.9\times 10^3}{n^2}\frac{(\log n)^{3/2}}{\lambda_{\min}^{11/2}}\Big(1 + \frac{1}{\Sigma^{5/2}}\Big),
\end{align*}
since $ \frac1{\Sigma^2} \le \frac15 + \frac4{5\Sigma^{5/2}}$.

\noindent Using the definition of regret $R_n(\mathcal A)=\max_kL_{k,n}-\frac{\Sigma}{n}$, we obtain

\begin{equation}\label{e:bound.reg.stime5}
R_n(\mathcal A_{CH}) \leq \frac{39\sqrt{\log(n)}}{n^{3/2} \lambda_{\min}^{5/2}} +  \frac{2.9\times 10^3}{n^2}\frac{(\log n)^{3/2}}{\lambda_{\min}^{11/2}}\Big(1 + \frac{1}{\Sigma}+\frac{1}{\Sigma^2} +  \frac{1}{\Sigma^{5/2}}\Big).
\end{equation}

\end{proof}



\section{Regret Bound for the Bernstein Algorithm}\label{app:B}

Let us consider $n>0$, $0<\de<1$ (that can be a function of $n$), $c_1>0$ and $c_2>0$  fixed.  We consider all the quantities considered in the definition of algorithm B-AS defined with respect to these fixed $n, \de, c_1, c_2$, and use the abbreviated notations $\hmu_{k,t}$, $\hsi_{k,t}$, $B_{k,t}$, $k_t$, and $T_{k,t}$.



\subsection{Basic Tools}\label{s:b-tools}

Before proving the bound in Theorems~\ref{thm:b-regret} and ~\ref{thm:b-regret-gauss} we need a number of technical tools, in particular for sub-Gaussian random variables.


The upper confidence bounds $B_{k,t}$ used in the B-AS algorithm is motivated by Theorem~10~in~\citep{maurer2009empirical}. We extend this result to sub-Gaussian random variables. We first restate Theorem~10 of \citep{maurer2009empirical}:

\begin{theorem}[\cite{maurer2009empirical}]\label{th:maurer}
Let $X_1,\ldots,X_t$ be $t \geq 2$ i.i.d.~random variables with variance $\var$ and mean $\mu$ and such that $\{X_i\}_{i=1}^t \in [0,b]$.
Then with probability at least $1 - \de$, we have

$$\Bigg|\sqrt{\frac{1}{t-1}\sum_{i=1}^t\Big(X_{i} - \frac{1}{t}\sum_{j=1}^t X_{j}\Big)^2} - \si\Bigg| \leq  b \sqrt{\frac{2\log(2/\de)}{t-1}}.$$
\end{theorem}

We now state and prove the following lemma (first statement of Lemma~\ref{l:b-upper-lower}).

\begin{lemma}\label{l:event-B-AS}
Let Assumption~\ref{a:subgaussian} holds, and $n\geq 2$, $c_1>0$, $c_2>0$, and $0 < \de < \min(1,c_2)$. For the event
%

\begin{equation}\label{eq:ucb-maurer-event-sub1}
\xi = \xi_{K,n}^B(\de) = \mathop{\bigcap_{1\leq k\leq K}}_{2\leq t\leq n}\left\lbrace \Bigg|\sqrt{\frac{1}{t-1}\sum_{i=1}^t\Big(X_{k,i} - \frac{1}{t}\sum_{j=1}^t X_{k,j}\Big)^2} - \si_k\Bigg| \leq 2a\sqrt{\frac{\log(2/\de)}{t}} \right\rbrace,
\end{equation}
%
where $a = 2\sqrt{c_1 \log(c_2/\de)} + \frac{\sqrt{c_1 \de (1+c_2 +\log(c_2/\de))}}{(1-\de)\sqrt{2\log(2/\de)}}n^{1/2}$, we have $\Pr[\xi] > 1-2nK\de$.

\end{lemma}

Note that the first term in the absolute value in Equation~\ref{eq:ucb-maurer-event-sub1} is the empirical standard deviation of arm $k$ computed as in Equation~\ref{eq:estim-var2} for $t$ samples. The event $\xi$ plays an important role in the proofs of this section and a number of statements will be proved on this event.

\begin{proof}

\noindent
\textbf{Step~1.~Truncating sub-Gaussian variables.}
%
We want to characterize the conditional mean and variance of the variables $X_{k,t}$ given that $|X_{k,t}-\mu_k| \leq \sqrt{c_1\log(c_2/\de)}$. For any non-negative random variable $Y$ and any $b\geq 0$, $\E[Y\ind{Y> b}] = \int_b^{\infty} \Prob[Y> \ep]d\ep + b \Prob[Y> b]$.\footnote{Let $\tilde Y = Y \ind{Y> b} + b\ind{Y\leq b}$, then $\E[\tilde Y] = \int_0^b \Prob[\tilde Y>\eps] d\eps + \int_b^{\infty} \Prob[\tilde Y>\eps]d\eps = b + \int_b^{\infty} \Prob[Y>\eps]d\eps$. Thus we can write $\E[Y\ind{Y> b}] = \E[\tilde Y] - b\Prob[Y \leq b] = \int_b^\infty \Prob[Y>\eps]d\eps + b\Prob[Y > b]$.} In order to simplify the notation we introduce the deviation random variable $S_{k,t} = X_{k,t}-\mu_k$. If we take $b = c_1 \log(c_2/\de)$ and use Assumption \ref{a:subgaussian}, we obtain $\mathbb P[S_{k,t}^2>b] \leq \de$ and

\begin{align*}
\E \Big[ S_{k,t}^2 \ind{S_{k,t}^2 > b} \Big] &= \int_{b}^{\infty} \Prob\big[S_{k,t}^2 > \ep\big] d\ep + b \Prob[S_{k,t}^2 > b]\leq \int_{b}^{\infty} c_2 \exp(-\ep/c_1) d\ep + b c_2 \exp(-b/c_1) \\ 
&= c_1 \de + c_1\de\log(c_2/\de)= c_1 \de\big(1+\log(c_2/\de)\big).
\end{align*}

By definition of $S_{k,t}$, we have $\E \big[ S_{k,t}^2 \I \{ S_{k,t}^2 > b\} \big] + \E\big[ S_{k,t}^2 \I\{S_{k,t}^2 \leq b\} \big] = \var_k$, which can be written as
\begin{align}
\frac{\E\big[S_{k,t}^2 \I\{S_{k,t}^2>b\}\big] - \si_k^2 \Prob\big[S_{k,t}^2>b\big]}{\Prob\big[S_{k,t}^2\leq b\big]} = \si_k^2 - \frac{\E\big[S_{k,t}^2 \I\{S_{k,t}^2 \leq b\}\big]}{\Prob\big[S_{k,t}^2\leq b\big]},
\end{align}
that combined with the previous equation, implies that
\begin{align}
\Big|\E\Big[ S_{k,t}^2 \big| S_{k,t}^2 \leq b \Big] - \si_k^2 \Big| &= \frac{\Big|\E \Big[ \big(S_{k,t}^2 - \si_k^2 \big)\ind{S_{k,t}^2 > b} \Big]\Big|}{\Prob\big[S_{k,t}^2 \leq b \big]} \nonumber \\
 &\leq \frac{c_1 \de (1+\log(c_2/\de)) + \de \si_k^2}{1-\de}, \label{truc3}
\end{align}
where we use $1+\log(c_2/\de)\geq 0$, that follows from $\de\leq c_2$.
Note also that Cauchy-Schwartz inequality implies
\begin{align*}
\Big| \E \big[ S_{k,t}\ind{S_{k,t}^2 > b} \big] \Big| &\leq  \sqrt{\E \big[ S_{k,t}^2 \I\{S_{k,t}^2 > b\}\big]} \\
&\leq \sqrt{c_1 \de (1+\log(c_2/\de))}.
\end{align*}
We now introduce the conditional mean of $X_{k,t}$ conditioned on small deviations, that is $\tilde{\mu}_k = \E \big[ X_{k,t} \big| S_{k,t}^2 \leq  b\big]  = \frac{\E[ X_{k,t} \I\{S_{k,t}^2 \leq b\} ]}{\Prob[S_{k,t}^2 \leq b]}$. Thus we can combine $\E \big[ X_{k,t} \I\{S_{k,t}^2 > b\} \big] + \E\big[ X_{k,t} \I\{S_{k,t}^2 \leq b\} \big] = \mu_k$ with the previous result and obtain
\begin{equation}\label{truc4}
|\tilde{\mu}_k - \mu_k| = \frac{\Big|\E \big[ S_{k,t}\I\{S_{k,t}^2 > b\} \big]\Big|}{\Prob\big[S_{k,t}^2 \leq b \big]} \leq \frac{\sqrt{c_1 \de (1+\log(c_2/\de))}}{1-\de}.
\end{equation}

We also define the variance of the conditional random variable $\tilde{\sigma}_k^2 = \V \big[ X_{k,t} | S_{k,t}^2 \leq  b\big] = \E\big[ S_{k,t}^2 |S_{k,t}^2 \leq b \big] - (\mu_k - \tilde{\mu_k})^2$. From Equations \ref{truc3} and \ref{truc4}, we derive
\begin{align*}
 |\tilde{\si}_k^2 - \var_k| &\leq  \Big|\E\big[ S_{k,t}^2 | S_{k,t}^2 \leq b \big] - \var_k \Big| + (\tilde{\mu}_k - \mu_k)^2\\
&\leq  \frac{c_1 \de (1+\log(c_2/\de)) + \de \si_k^2}{1-\de} + \frac{c_1 \de (1+\log(c_2/\de))}{(1-\de)^2}\\
&\leq \frac{2c_1 \de (1+\log(c_2/\de)) + \de\var_k}{(1-\de)^2}.
\end{align*}
In order to get the final result, we first bound the variance $\si_k^2$ as a function of the constants $c_1$ and $c_2$ using the sub-Gaussian assumption as
\begin{equation}\label{eq:variance-ub}
\si_k^2 = \E[(X_{k,t}-\mu_k)^2] = \int_0^\infty \Prob[X_{k,t}-\mu_k)^2 > \eps] d\eps \leq \int_0^\infty c_2 \exp(-\eps / c_1) d\eps = c_1 c_2.
\end{equation}
Finally, using $\sqrt{|x^2-y^2|} \geq |x-y|$ for $x,y \geq 0$, we obtain
\begin{equation}\label{truc2}
|\tilde{\si}_k - \si_k| \leq \frac{\sqrt{2c_1 \de (1+c_2 +\log(c_2/\de))}}{1-\de}.
\end{equation}

\noindent
\textbf{Step~2.~Application of large deviation inequalities.}

Let $\xi_1 = \xi_{1,K,n}(\de)$ be the event:
\begin{equation*}\label{eq:sub.gaus}
\xi_1 = \bigcap_{1\leq k\leq K,\;1\leq t\leq n}\left\lbrace |X_{k,t} - \mu_k| \leq \sqrt{c_1 \log(c_2/\de)} \right\rbrace.
\end{equation*}
Under Assumption~\ref{a:subgaussian}, using a union bound, we have that the probability of this event is at least $1-nK\de$.
On $\xi_1$, the $\{X_{k,i}\}_i,\;1\leq k\leq K,\;1\leq i\leq t$ are $t$ i.i.d.~bounded random variables with standard deviation $\tilde{\si}_k$.

Let $\xi_2 = \xi_{2,K,n}(\de)$ be the event:

\begin{equation*}
\xi_2 = \bigcap_{1\leq k\leq K,\;2\leq t\leq n}\left\lbrace \Bigg|\sqrt{\frac{1}{t-1}\sum_{i=1}^t\Big(X_{k,i} - \frac{1}{t}\sum_{j=1}^t X_{k,j}\Big)^2} - \tilde{\si}_k\Bigg| \leq 2\sqrt{c_1 \log(c_2/\de)}\sqrt{2\frac{\log(2/\de)}{t-1}} \right\rbrace.
\end{equation*}

Using Theorem~\ref{th:maurer} and a union bound, we deduce that $ \Pr[\xi_1 \cap \xi_2] \geq 1-2nK\de$. Now, from Equation~\eqref{truc2}, we have on $\xi_1 \cap \xi_2$, for all $1\leq k\leq K,\;2\leq t\leq n$:
\begin{align*}
\Bigg|\sqrt{\frac{1}{t-1}\sum_{i=1}^t\Big(X_{k,i} - \frac{1}{t}\sum_{j=1}^t X_{k,j}\Big)^2} &- \si_k\Bigg| \\
& \leq 2\sqrt{c_1 \log(c_2/\de)}\sqrt{\frac{2\log(2/\de)}{t-1}}+ \frac{\sqrt{2c_1 \de (1+c_2 +\log(c_2/\de))}}{1-\de}\\
&\leq 4\sqrt{c_1 \log(c_2/\de)}\sqrt{\frac{\log(2/\de)}{t}}+ \frac{\sqrt{2c_1 \de (1+c_2 +\log(c_2/\de))}}{1-\de},
\end{align*}
from which we deduce Lemma~\ref{l:event-B-AS} (since $\xi_1 \cap \xi_2 \subseteq \xi$ and $2 \leq t\leq n$).
\end{proof}

We transcribe the definition~\eqref{eq:ucb-maurer-event-sub1} of $\xi$ in the last lemma into the following lemma when the number of samples $T_{k,t}$ are random.

\begin{lemma}\label{l:bernstein-var-sub}
For $t=2K,\dots,n$, let $T_{k,t}$ be any random variable taking values in $\{2,\dots,n\}$. Let $\hvar_{k,t}$ be the empirical variance computed from Equation~\eqref{eq:estim-var2}. Then, on the event $\xi$, we have:
\begin{equation}\label{eq:ucb-maurer-event-sub2}
|\hsi_{k,t} - \si_k| \leq 2a\sqrt{\frac{\log(2/\de)}{T_{k,t}}}\;,
\end{equation}
where $a = 2\sqrt{c_1 \log(c_2/\de)}+ \frac{\sqrt{T_{k,t}c_1 \de (1+c_2 +\log(c_2/\de))}}{(1-\de)\sqrt{2\log(2/\de)}}$
\end{lemma}

\subsection{Allocation Performance}\label{s:b-allocation}

In this section, we first provide the proof of Lemma~\ref{l:b-upper-lower}, we then derive the regret bound of Theorem~\ref{thm:b-regret} in the general case, and we prove Theorem~\ref{thm:b-regret-gauss} for Gaussians. \\

Recall that $n \geq 5K$. This will be useful in the following.\\

%
%

\begin{proof}[Proof of Lemma~\ref{l:b-upper-lower}]

Note first that the first part of the claim of the lemma is exactly Lemma~\ref{l:event-B-AS}. The rest of the proof consists of the following five main steps. Until the end of the proof, we assume that $\xi$ holds. \\

\noindent
{\bf Step~1.~Lower bound of order $\Omega(\sqrt{n})$.} We first recall for any arm $q$ the definition of $B_{q,t+1}$ used in the B-AS algorithm

\begin{small}
\begin{equation*}
B_{q,t+1} = \frac{1}{T_{q,t}} \Bigg( \hsi_{q,t} + 2a\sqrt{\frac{\log(2/\de)}{T_{q,t}}}\Bigg)^2.
\end{equation*}
\end{small}

\noindent Using Lemma~\ref{l:bernstein-var-sub} it follows that on $\xi$, for any $q$ such that $T_{q,t} \geq 2$,

\begin{small}
\begin{equation}\label{e:bound.B}
 \frac{\var_q}{T_{q,t}} \leq B_{q,t+1} \leq \frac{1}{T_{q,t}}\Bigg(\si_{q} + 4a \sqrt{\frac{\log(2/\de)}{T_{q,t}}} \Bigg)^2.
\end{equation}
\end{small}

Let $q$ be the index of an arm such that $T_{q,n} \geq \frac{n}{\narms}$ and $t+1\leq n$ be the last time that it was pulled, i.e.,~$T_{q,t}=T_{q,n}-1$ and $T_{q,t+1}=T_{q,n}$.~\footnote{Note that such an arm always exists for any possible allocation strategy given the constraint $n=\sum_{p}T_{p,n}$.} From Equation~\eqref{e:bound.B} and the fact that $T_{q,n} \geq \frac{n}{\narms}\geq 5$ (see condition on $c(\de)$, and also the beginning of this section) and $T_{q,t} \geq 3$, we obtain on $\xi$

\begin{small}
\begin{equation}\label{e:meca.lb}
B_{q,t+1} \leq \frac{1}{T_{q,t}}\Bigg(\si_{q} + 4a \sqrt{\frac{\log(2/\de)}{T_{q,t}}} \Bigg)^2 \leq \frac{4K}{3n}\Big(\sqrt{\Sigma} + 4a \sqrt{\frac{\log(2/\de)}{3}}\Big)^2,
\end{equation}
\end{small}

\noindent where we also used $T_{q,n}\geq 4$ to bound $T_{q,t}$ in the parenthesis and the fact that $\si_q \leq \sqrt{\Sigma}$. Since at time $t+1$ we assumed that arm $q$ has been chosen then for any other arm $p$, we have

\begin{small}
\begin{equation}\label{e:meca}
B_{p,t+1} \leq B_{q,t+1}.
\end{equation}
\end{small}
\noindent From the definition of $B_{p,t+1}$, removing all the terms but the last and using the fact that $T_{p,t} \leq T_{p,n}$, we obtain the lower bound

\begin{small}
\begin{equation}\label{e:meca.hb}
B_{p,t+1} \geq 4 a^2 \frac{\log(2/\de)}{T_{p,t}^2} \geq 4 a^2 \frac{\log(2/\de)}{T_{p,n}^2}\;.
\end{equation}
\end{small}

\noindent Combining Equations~\ref{e:meca.lb}--\ref{e:meca.hb}, we obtain

\begin{small}
\begin{equation*}
4 a^2 \frac{\log(2/\de)}{T_{p,n}^2} \leq \frac{4K\Big(\sqrt{\Sigma} + 3a \sqrt{\log(2/\de)}\Big)^2} {3n}.
\end{equation*}
\end{small}

\noindent Finally, this implies that for any $p$

\begin{small}
\begin{equation}\label{eq:lb.vloose}
T_{p,n} \geq \frac{2 a \sqrt{\log(2/\de)} }{\sqrt{\Sigma} + 3a \sqrt{\log(2/\de)}} \sqrt{\frac{3n}{4K}}.
\end{equation}
\end{small}

\noindent In order to simplify the notation, in the following we use

\begin{small}
\begin{equation*}
c(\de) = \frac{a \sqrt{3\log(2/\de)} }{\sqrt{K}\Big(\sqrt{\Sigma} + 3a \sqrt{\log(2/\de)}\Big)},
\end{equation*}
\end{small}

\noindent thus obtaining $T_{p,n} \geq c(\de) \sqrt{n}$ on the event $\xi$ for any $p$. \\

\noindent
{\bf Step~2.~Mechanism of the algorithm.} Note that as $n \geq 5 K$, there is at least an arm $q$ that is pulled after initialization. Let, for such an arm $q$, $t+1>2K$ be the time when arm $q$ is pulled for the last time, that is $T_{q,t} = T_{q,n}-1 \geq 2$. Since at time $t+1$ this arm $q$ is chosen, then for any other arm $p$, we have

\begin{small}
\begin{equation}\label{e:meca2}
B_{p,t+1} \leq B_{q,t+1}\;.
\end{equation}
\end{small}

\noindent From Equation~\eqref{e:bound.B} and $T_{q,t} = T_{q,n} -1$, we obtain

\begin{small}
\begin{equation}\label{e:meca.lb2}
B_{q,t+1} \leq \frac{1}{T_{q,t}}\Bigg(\si_{q} + 4 a \sqrt{\frac{\log(2/\de)}{T_{q,t}}} \Bigg)^2 = \frac{1}{T_{q,n}-1}\Bigg(\si_{q} + 4 a \sqrt{\frac{\log(2/\de)}{T_{q,n}-1}} \Bigg)^2.
\end{equation}
\end{small}

\noindent Furthermore, since $T_{p,t} \leq T_{p,n}$ and $T_{p,t} \geq 2$ (as $t \geq 2K$), then

\begin{small}
\begin{equation}\label{e:meca.hb2}
B_{p,t+1} \geq \frac{\var_{p}}{T_{p,t}} \geq \frac{\var_{p}}{T_{p,n}}.
\end{equation}
\end{small}

\noindent Combining Equations~\ref{e:meca2}--\ref{e:meca.hb2}, we obtain

\begin{small}
\begin{equation*}
\frac{\var_{p}}{T_{p,n}}(T_{q,n}-1) \leq \Bigg(\si_{q} + 4a \sqrt{\frac{\log(2/\de)}{T_{q,n}-1}} \Bigg)^2.
\end{equation*}
\end{small}

\noindent Summing over all $q$ that are pulled after initialization on both sides, we obtain on $\xi$ for any arm $p$

\begin{align}\label{eq:main.lb}
\frac{\var_{p}}{T_{p,n}}(n - 2K) \leq \sum_{q|T_{q,n}>2} \Bigg(\si_{q} + 4a \sqrt{\frac{\log(2/\de)}{T_{q,n}-1}} \Bigg)^2,
\end{align}

\noindent because the arms that are not pulled after the initialization are only pulled twice (so $\sum_{q|T_{q,n}>2} (T_{q,n}-1)\geq n-2K$).

\noindent
{\bf Step~3.~Intermediate lower bound.} It is possible to rewrite Equation~\eqref{eq:main.lb}, using the fact that $T_{q,n} \geq 2$, as

\begin{equation}\label{eq:newandras}
\frac{\var_{p}}{T_{p,n}}(n - 2K) \leq \sum_q \Bigg(\si_{q} + 4a \sqrt{\frac{\log(2/\de)}{T_{q,n}-1}} \Bigg)^2 \leq \sum_q \Bigg(\si_{q} + 4a \sqrt{\frac{2\log(2/\de)}{T_{q,n}}} \Bigg)^2.
\end{equation}

\noindent Plugging Equation~\eqref{eq:lb.vloose} in Equation~\eqref{eq:newandras}, we have on $\xi$ for any arm $p$

\begin{small}
\begin{equation}\label{eq:intermediate-bound}
\frac{\var_{p}}{T_{p,n}}(n - 2K) \leq \sum_q \Bigg(\si_{q} + 4a \sqrt{\frac{2\log(2/\de)}{T_{q,n}}} \Bigg)^2 \leq \Bigg(\sqrt{\Sigma} + 4\sqrt{K}a \sqrt{2\frac{\log(2/\de)}{c(\de)\sqrt{n}}} \Bigg)^2,
\end{equation}
\end{small}

\noindent because for any sequence $(a_k)_{i=1,\ldots,K}\geq 0$, and any $b \geq 0$, $\sum_k (a_k + b)^2 \leq (\sqrt{\sum_k a_k^2} + \sqrt{K}b)^2$ by Cauchy-Schwartz.

\noindent Building on this bound we shall recover the desired bound. \\

\noindent
{\bf Step~4.~Final lower bound.} We first expand the square in Equation~\eqref{eq:newandras} using $T_{q,n}\geq 2$ as

\begin{small}
\begin{equation*}
\frac{\var_{p}}{T_{p,n}}(n - 2K) \leq \sum_q\si_q^2 + 8a\sqrt{2\log(2/\de)}\sum_q \frac{\si_q}{\sqrt{T_{q,n}}} + \sum_q \frac{32a^2\log(2/\de)}{T_{q,n}}.
\end{equation*}
\end{small}

\noindent We now use the bound in Equation~\eqref{eq:intermediate-bound} in the second term of the RHS and the bound in Equation~\eqref{eq:lb.vloose} to bound $T_{k,n}$ in the last term, thus obtaining

\begin{small}
\begin{equation*}
\frac{\var_{p}}{T_{p,n}}(n - 2K) \leq \Sigma + 8a\sqrt{2\log(2/\de)}\frac{K}{\sqrt{n-2K}}\Bigg(\sqrt{\Sigma} + 4\sqrt{K}a \sqrt{2\frac{\log(2/\de)}{c(\de)\sqrt{n}}}  \Bigg) + \frac{32Ka^2\log(2/\de)}{c(\de)\sqrt{n}}.
\end{equation*}
\end{small}

\noindent By using again $n\geq 5 K$ and some algebra, we get

\begin{small}
\begin{align}\label{eq:almost-lower-bound}
\frac{\var_{p}}{T_{p,n}}&(n - 2K) \leq \Sigma + 16a\sqrt{\log(2/\de)}\frac{K}{\sqrt{n}}\Bigg(\sqrt{\Sigma} + 4\sqrt{K}a \sqrt{2\frac{\log(2/\de)}{c(\de)\sqrt{n}}}  \Bigg) + \frac{32Ka^2\log(2/\de)}{c(\de)\sqrt{n}} \nonumber\\
&\leq \Sigma + 16Ka \sqrt{\frac{\Sigma\log(2/\de)}{n}} + 64\sqrt{2}K^{3/2}a^2 \frac{\log(2/\de)}{\sqrt{c(\de)}}n^{-3/4} + \frac{32Ka^2\log(2/\de)}{c(\de)\sqrt{n}} \nonumber \\
&= \Sigma + \frac{16Ka\sqrt{\log(2/\de)}}{\sqrt{n}} \Bigg(\sqrt{\Sigma} + \frac{2a\sqrt{\log(2/\de)}}{c(\de)} \Bigg) + 64\sqrt{2}K^{3/2}a^2 \frac{\log(2/\de)}{\sqrt{c(\de)}}n^{-3/4}.
\end{align}
\end{small}

\noindent We now invert the bound and obtain the final lower bound on $T_{p,n}$ as follows:

\begin{small}
\begin{align*}
&T_{p,n} \geq \frac{\sigma_p^2 (n\!-\!3K)}{\Sigma} \Bigg[1 + \frac{16Ka\sqrt{\log(2/\de)}}{\Sigma\sqrt{n}} \Bigg(\sqrt{\Sigma} + \frac{2a\sqrt{\log(2/\de)}}{c(\de)} \Bigg) + 64\sqrt{2}K^{3/2}a^2 \frac{\log(2/\de)}{\Sigma\sqrt{c(\de)}}n^{-3/4}\Bigg]^{-1} \\
&\geq \frac{\sigma_p^2 (n-2K)}{\Sigma} \Bigg[1 - \frac{16Ka\sqrt{\log(2/\de)}}{\Sigma\sqrt{n}} \Bigg(\sqrt{\Sigma} + \frac{2a\sqrt{\log(2/\de)}}{c(\de)} \Bigg) - 64\sqrt{2}K^{3/2}a^2 \frac{\log(2/\de)}{\Sigma\sqrt{c(\de)}}n^{-3/4}\Bigg] \\
&\geq T_{p,n}^* - K\lambda_p \Bigg[ \frac{16a\sqrt{\log(2/\de)}}{\Sigma} \Bigg(\sqrt{\Sigma} + \frac{2a\sqrt{\log(2/\de)}}{c(\de)} \Bigg)n^{1/2} + 64\sqrt{2K}a^2 \frac{\log(2/\de)}{\Sigma\sqrt{c(\de)}}\;n^{1/4} + 2\Bigg].
\end{align*}
\end{small}

\noindent Note that the above lower bound holds on $\xi$ for any arm $p$. \\

\noindent
{\bf Step~5.~Upper bound.} The upper bound on $T_{p,n}$ follows by using $T_{p,n} = n - \sum_{q\neq p} T_{q,n}$ and the previous lower bound, that is

\begin{small}
\begin{align*}
&T_{p,n} \leq n - \sum_{q\neq p} T_{q,n}^* \\
&\quad+ \sum_{q\neq p} K\lambda_q \Bigg[ \frac{16a\sqrt{\log(2/\de)}}{\Sigma} \Bigg(\sqrt{\Sigma} + \frac{2a\sqrt{\log(2/\de)}}{c(\de)} \Bigg)n^{1/2} + 64\sqrt{2K}a^2 \frac{\log(2/\de)}{\Sigma\sqrt{c(\de)}}\;n^{1/4} + 2\Bigg] \\
&\leq T_{p,n}^* + K\Bigg[ \frac{16a\sqrt{\log(2/\de)}}{\Sigma} \Bigg(\sqrt{\Sigma} + \frac{2a\sqrt{\log(2/\de)}}{c(\de)} \Bigg)n^{1/2} + 64\sqrt{2K}a^2 \frac{\log(2/\de)}{\Sigma\sqrt{c(\de)}}\;n^{1/4} + 2\Bigg].
\end{align*}
\end{small}

\end{proof}


\subsection{Regret Bounds}\label{s:b-regret}

With the allocation performance, we now move to the regret bound showing how the number of pulls translates into the losses $L_{kn}$ and the global regret as stated in Theorem~\ref{thm:b-regret}.

We first state some technical results.

\subsubsection{Bound on the Regret Outside $\xi$}

The next lemma provides a bound for the loss whenever the event $\xi$ does not hold.

\begin{lemma}\label{l:some-lemma}
Let Assumption~\ref{a:subgaussian} holds. If $2nK\de< c_2$, then for every arm $k$, we have\footnote{Note that for $\de = n^{-7/2}$,
$n\geq 5K$, and $c_2\ge 1$, we have $2nK\de = 2Kn^{-5/2} < c_2$.}
\begin{equation*}
\E\big[(\hmu_{k,n}-\mu_k)^2 \I\{\xi^C\}\big] \leq 2c_1n^2K\de (1 + \log(c_2/2nK\de)).
\end{equation*}
\end{lemma}
\begin{proof}
Since the arms have sub-Gaussian distribution, for any $1\leq k\leq K$ and $1\leq t\leq n$, we have
\begin{equation*}
\Prob\big[(X_{k,t}-\mu_k)^2 \geq \ep\big] \leq c_2 \exp(-\ep/c_1)\;,
\end{equation*}
and thus since $c_2>2nK\de$, we obtain 
\begin{equation*}
\Prob\big[(X_{k,t}-\mu_k)^2 \geq c_1 \log(c_2/2nK\de)\big] \leq 2nK\de\;.
\end{equation*}
Since $\Prob[\xi^C] \leq 2nK\de$, the previous equation implies, using $c_2/(2nK\de)>1$
\begin{align*}
\E\big[(X_{k,t}-\mu_k)^2\ind{\xi^C}\big] &= \int_{0}^{\infty} \Prob\big[(X_{k,t}-\mu_k)^2 \ind{\xi^C} > \ep\big] d\ep\\
&\leq \int_{c_1 \log(c_2/2nK\de)}^{\infty} c_2 \exp(-\ep/c_1) d\ep + c_1 \log(c_2/2nK\de) \Prob[\xi^C]\\
&\leq 2c_1nK\de (1 + \log(c_2/2nK\de))\;.
\end{align*}
The claim follows from the fact that $\E\big[(\hmu_{k,n}-\mu_k)^2 \I\{\xi^C\}\big]\leq  \sum_{t=1}^{n}\E\big[(X_{k,n}-\mu_k)^2 \I\{\xi^C\}\big] \leq  2c_1n^2K\de (1 + \log(c_2/2nK\de))$.
\end{proof}

\subsubsection{Other Technical Inequalities}\label{app:tr}

At first let us write, for the sake of convenience,

\begin{small}
\begin{align*}
B = 16Ka\sqrt{\log(2/\de)} \Bigg(\sqrt{\Sigma} + \frac{2a\sqrt{\log(2/\de)}}{c(\de)} \Bigg)\;\hspace{0.5cm} \operatorname{and} \hspace{0.5cm} C = 64\sqrt{2}K^{3/2}a^2 \frac{\log(2/\de)}{\sqrt{c(\de)}}.
\end{align*}
\end{small}

\paragraph{Upper and lower bound on $a$} If $\de = n^{-7/2}$, with $n \geq 5K \geq 10$ and $c_2 \geq 1$
\begin{align*}
a &= 2\sqrt{c_1 \log(c_2/\de)} + \frac{\sqrt{c_1 \de (1+c_2 +\log(c_2/\de))}}{(1-\de)\sqrt{2\log(2/\de)}}n^{1/2}\\
&\leq \sqrt{14c_1 (c_2+1) \log(n)} + \frac{2}{n^{5/4}}\sqrt{c_1 (1+c_2)} \leq \sqrt{15c_1 (c_2+1) \log(n)}\\
 &\leq 4\sqrt{c_1 (c_2+1) \log(n)}.
\end{align*}

We also have by just keeping the first term, since $c_2 \geq 1$
\begin{align*}
a &= 2\sqrt{c_1 \log(c_2/\de)} + \frac{\sqrt{c_1 \de (1+c_2 +\log(c_2/\de))}}{(1-\de)\sqrt{2\log(2/\de)}}n^{1/2}\geq 2\sqrt{c_1} \geq \sqrt{c_1}.
\end{align*}

\paragraph{Lower bound on $c(\de)$ when $\de = n^{-7/2}$}

See Lemma~\ref{l:b-upper-lower} for the definition of $c(\de)$. Using the fact that the arms have sub-Gaussian distribution we showed in Equation~\eqref{eq:variance-ub} that $\sigma_k^2 \leq c_1 c_2$, then we also have $\Sigma \leq Kc_1c_2$. If $\de = n^{-7/2}$, we obtain by using the previous lower bound on $a$ that

\begin{small}
\begin{align*}
c(\de = n^{-7/2}) &= \frac{a \sqrt{3\log(2/\de)} }{\sqrt{3K}\Big(\sqrt{\Sigma/3} + a \sqrt{3\log(2/\de)}\Big)}= \frac{1}{\sqrt{3K}} \Bigg(1-\frac{\sqrt{\Sigma/3}}{\sqrt{\Sigma/3} + a\sqrt{\log 2/\de}} \Bigg)\\
&\geq \frac{1}{\sqrt{3K}} \Bigg(1-\frac{\sqrt{\Sigma/3}}{\sqrt{\Sigma/3} + \sqrt{c_1 \log 2/\de}} \Bigg)\geq \frac{1}{\sqrt{3K}} \Bigg(1-\frac{\sqrt{\Sigma/3}}{\sqrt{\Sigma/3} + \sqrt{c_1}} \Bigg)\geq \frac{1}{\sqrt{K}} \Bigg(\frac{1}{\sqrt{Kc_2} + \sqrt{3}} \Bigg)
\end{align*}
\end{small}

\noindent by using $\Sigma \leq Kc_2c_1$ for the last step.

\paragraph{Upper bound on the loss outside $\xi$ when $\de = n^{-7/2}$}


We get from Lemma~\ref{l:some-lemma} when $\de = n^{-7/2}$, when $c_2 \geq 1$ and when $n\geq 5K$ that

\begin{align*}
\E\big[(\hmu_{k,n}-\mu_k)^2 \ind{\xi^C}\big] &\leq 2c_1n^2K\de \Big(1 + \log\big(\frac{c_2}{2nK\de}\big)\Big)\leq 2c_1Kn^{-3/2}\Big(1+(c_2+1)\log\big(\frac{n^{5/2}}{2K}\big)\Big)\\
&\leq 2c_1Kn^{-3/2}\big(1+\frac{5}{2}(c_2+1)\log(n)\big)\leq 7c_1K(c_2+1)\log(n)n^{-3/2}.
\end{align*}

\paragraph{Upper bound on $B$ for $\de = n^{-7/2}$}

See the proof of Theorem~\ref{thm:b-regret} for the definition of $B$ (the notation $B$ we use in this section is for technical purposes and has nothing to do with the $B$ introduced in the proofs for algorithm CH-AS). When $\de = n^{-7/2}$, when $c_2 \geq 1$ and when $n\geq 5K\geq 10$,

\begin{align*}
 B &= 16Ka\sqrt{\log(2/\de)} \Bigg(\sqrt{\Sigma} + \frac{2a\sqrt{\log(2/\de)}}{c(\de)} \Bigg)\\
&\leq 16Ka\sqrt{7/2\log(2n)} \Big(\sqrt{\Sigma} + 2\sqrt{K} (\sqrt{\Sigma} + 3a \sqrt{7/2 \log(2n)}) \Big)\\
&\leq 16Ka\sqrt{7/2\log(2n)} \Big(\sqrt{\Sigma} + 2\sqrt{K\Sigma} + 12\sqrt{K} \sqrt{c_1 (c_2+1) 7 \log(n)\log(2n)} \Big)\\
&\leq 16Ka\sqrt{7/2\log(2n)} \Big(3K\sqrt{c_1c_2} + 45\sqrt{K} \sqrt{c_1 (c_2+1)}\log(n) \Big)\\
&\leq 32K\sqrt{14c_1(c_2+1)\log n \log(2n)} \Big(48K\sqrt{c_1(c_2+1)}\log(n) \Big)\\
&\leq 8\times10^3 K^2 c_1(c_2+1)\log^2(n).
\end{align*}

\paragraph{Upper bound on $C$ for $\de = n^{-7/2}$}

See the proof of Theorem~\ref{thm:b-regret} for the definition of $C$. When $\de = n^{-7/2}$, when $c_2 \geq 1$ and when $n\geq 5K\geq 10$,

\begin{align*}
C&= 64\sqrt{2}K^{3/2}a^2 \frac{\log(2/\de)}{\sqrt{c(\de)}}= 64\sqrt{2} K^{3/2} \frac{a^2 \log (2/\de)}{\sqrt{a} (3\log (2/\de))^{1/4}} K^{1/4} (\sqrt{\Sigma} + 3a\sqrt{\log (2/\de)})^{1/2}\\
&\leq 64\sqrt{2} K^{3/2} a^{3/2} (\log (2/\de))^{3/4} \frac{1}{3^{1/4}} K^{1/4} (\sqrt{Kc_1c_2} + 12\sqrt{c_1(c_2+1) \log n} \sqrt{7 \log n})^{1/2}\\
&\leq 128\sqrt{2} \frac{1}{3^{1/4}} K^{7/4} (2\sqrt{2c_1(c_2+1)\log n})^{3/2} (7\log n)^{3/4}  \sqrt{24}K^{1/4} (c_1(c_2+1))^{1/4} \sqrt{\log n} \\
&\leq 14\times 10^3 K^2 c_1(c_2+1) \log^2(n).
\end{align*}

We are now ready to prove Theorem~\ref{thm:b-regret}.

%

\begin{proof} [Proof of Theorem~\ref{thm:b-regret}]

\noindent Equation~\eqref{eq:almost-lower-bound} becomes using the constants $B,C$ that we introduced

\begin{small}
\begin{equation}\label{eq:allowb.mod}
 \frac{\var_p}{T_{p,n}}(n-2K) \leq \Sigma + \frac{B}{\sqrt{n}} + \frac{C}{n^{3/4}}.
\end{equation}
\end{small}

\noindent We also have the upper bound in Lemma~\ref{l:b-upper-lower} which can be rewritten:

\begin{small}
\begin{equation*}
 T_{p,n} \leq T_{p,n}^* + \frac{B}{\Sigma}\sqrt{n} + \frac{C}{\Sigma}n^{1/4} + 2K.
\end{equation*}
\end{small}

\noindent Note that because this upper bound holds on an event of probability bigger than $1-4nK\de$ and also because $T_{p,n}$ is bounded by $n$ anyways, we can convert the former upper bound in a bound in expectation:

\begin{small}
\begin{equation} \label{eq:exp.bound.mod}
 \E[T_{p,n}] \leq T_{p,n}^* +  \frac{B}{\Sigma}\sqrt{n} + \frac{C}{\Sigma}n^{1/4} + 2K + n \times 4nK\de.
\end{equation}
\end{small}

\noindent We recall that the loss of any arm $k$ is decomposed in two parts as follows:

\begin{small}
\begin{equation*}
L_{k,n} = \E[(\hmu_{k,n}-\mu_k)^2 \ind{\xi}] + \E[(\hmu_{k,n}-\mu_k)^2 \mathbb I\{\xi^C\}].
\end{equation*}
\end{small}

\noindent By combining the fact that $T_{k,n}$ is again a stopping time with Equations \ref{eq:allowb.mod}, \ref{eq:exp.bound.mod}, and \ref{eq:wald-inequality} (as done in Equation~\eqref{eq:st2.1}), and since $n-2K>0$, we obtain for the first part of the loss:

\begin{small}
\begin{align*}
&\E[(\hmu_{k,n}-\mu_k)^2 \ind{\xi}] \\
\leq& \frac{1}{\var_k(n-2K)^2} \Big( \Sigma + \frac{B}{\sqrt{n}} + \frac{C}{n^{3/4}} \Big)^2 \Big( T_{k,n}^*+ \frac{B}{\Sigma}\sqrt{n} + \frac{C}{\Sigma}n^{1/4} + 2K + 4n^2K\de \Big) \\
\leq& \frac{1}{(n-2K)^2} \Bigg( \Sigma^2 +2\Sigma(\frac{B}{\sqrt{n}} + \frac{C}{n^{3/4}}) + \frac{(B+C)^2}{n} \Bigg) \Big( \frac{n}{\Sigma}+ \frac{B}{\Sigma^2\lambda_k}\sqrt{n} + \frac{C}{\Sigma^2\lambda_k}n^{1/4}+\frac{2K}{\Sigma\lambda_k} + \frac{4n^2K\de}{\Sigma\lambda_k} \Big)\\
\leq& \frac{1}{(n-2K)^2}\Bigg(n\Sigma + \frac{B}{\lambda_k}\sqrt{n} + \frac{C + 2K\Sigma}{\lambda_k}n^{1/4}+  \frac{4n^2K\Sigma\de}{\lambda_k} + 2B\sqrt{n} + 2Cn^{1/4}\\
 &+ \frac{2 (B+C)(\frac{B}{\Sigma} + \frac{C}{\Sigma} + 2K)}{\lambda_k} + \frac{8(B+C)n^{3/2}K\de}{\lambda_k} + (B+C)^2 \Big(\frac{1}{\Sigma} + \frac{B+C}{\Sigma^2 \lambda_k} + \frac{2K}{\Sigma \lambda_k} \Big) +  4nK\de\frac{(B+C)^2}{\Sigma\lambda_k}\Bigg)\\
=& \frac{1}{(n-2K)^2}\Bigg(n\Sigma + (\frac{B}{\lambda_k} + 2B)\sqrt{n} + (\frac{C + 2K\Sigma}{\lambda_k} + 2C)n^{1/4}\\ 
&+  \frac{2 (B+C)(\frac{B+C}{\Sigma}  + 2K)}{\lambda_k} + (B+C)^2 \Big(\frac{1}{\Sigma} + \frac{B+C}{\Sigma^2 \lambda_k} + \frac{2K}{\Sigma \lambda_k} \Big)\\
 &+ \frac{4n^2K\Sigma\de}{\lambda_k}  + \frac{8(B+C)n^{3/2}K\de}{\lambda_k} +  4nK\de\frac{(B+C)^2}{\Sigma\lambda_k}\Bigg)\\
\leq& \frac{1}{(n-2K)^2}\Bigg(n\Sigma + \frac{3B}{\lambda_k}\sqrt{n} + \frac{3C + 2K\Sigma}{\lambda_k}n^{1/4}\\ 
&+  \frac{K (B+C)^3}{\lambda_k} \Big(\frac{2}{K\Sigma(B+C)} + \frac{4}{(B+C)^2} + \frac{\lambda_k}{K\Sigma(B+C)} + \frac{1}{\Sigma^2K} + \frac{2}{\Sigma (B+C)} \Big)\\
 &+  \frac{4\de n^2K}{\lambda_k}\Big( \Sigma + 2(B+C) +  \frac{(B+C)^2}{\Sigma}\Big) \Bigg),
\end{align*}
\end{small}
and since $B+C \geq 2$ for $\de=n^{-7/2}$, $n\ge 16K/3 \geq 8$, it implies
\begin{small}
\begin{align*}
&\E[(\hmu_{k,n}-\mu_k)^2 \ind{\xi}] \\
\leq& \frac{1}{(n-2K)^2}\Bigg(n\Sigma + \frac{3B}{\lambda_k}\sqrt{n} + \frac{3C + 2K\Sigma}{\lambda_k}n^{1/4} +  \frac{K (B+C)^3}{\lambda_k} \Big(\frac{1}{2\Sigma} + 1 + \frac{1}{8\Sigma} + \frac{1}{2\Sigma^2} + \frac{1}{\Sigma} \Big)\\
 &+  \frac{4\de n^2K}{\lambda_k}\Big( \Sigma + 2(B+C) +  \frac{(B+C)^2}{\Sigma}\Big) \Bigg)\\
\leq& \frac{1}{(n-2K)^2}\Bigg(n\Sigma + \frac{3B}{\lambda_k}\sqrt{n} + \frac{3C + 2K\Sigma}{\lambda_k}n^{1/4} +  \frac{K (B+C)^3}{\lambda_k} \Big(\frac{1}{2\Sigma^2} + \frac{13}{8}\Sigma + 1 \Big)\\
 &+  \frac{4\de n^2K}{\lambda_k}\Big( \Sigma + 2(B+C) +  \frac{(B+C)^2}{\Sigma}\Big) \Bigg)\\
\leq& \frac{1}{(n-2K)^2}\Bigg(n\Sigma + \frac{3B}{\lambda_k}\sqrt{n} + \frac{3C + 2K\Sigma}{\lambda_k}n^{1/4}+ K \frac{(B+C)^3}{\lambda_k}(\frac{1}{\Sigma^2} +8) \\
 &+  \frac{4\de n^2K}{\lambda_k}\Big( \Sigma + 2(B+C) +  \frac{(B+C)^2}{\Sigma}\Big)\Bigg).
\end{align*}
\end{small}

\noindent Now note that, as $\de=n^{-7/2}$ and $n\ge 4K$
\begin{small}
\begin{align*}
\E[(\hmu_{k,n}-\mu_k)^2& \ind{\xi}]  \leq \frac{1}{(n-2K)^2}\Bigg(n\Sigma + \frac{3B}{\lambda_k}\sqrt{n} + \frac{3C + 2K\Sigma}{\lambda_k}n^{1/4}+  K \frac{(B+C)^3}{\lambda_k}(\frac{1}{\Sigma^2} +8) +  \frac{4K \Sigma}{n^{3/2} \lambda_k}\Big( 1+ \frac{B+C}{\Sigma}\Big)^2\Bigg)\\
&\leq \Bigg(\frac{1}{n^2} + \frac{8K}{n^3}\Bigg)\Bigg(n\Sigma + \frac{3B}{\lambda_k}\sqrt{n} + \frac{3C + 2K\Sigma}{\lambda_k}n^{1/4}+   K \frac{(B+C)^3}{\lambda_k}(\frac{1}{\Sigma^2} +8) +  \frac{8K \Sigma}{n^{3/2} \lambda_k}(B+C)^2( 1+ \frac{1}{\Sigma^2})\Bigg)\\
&\leq \frac{\Sigma}{n} + \frac{8K\Sigma}{n^2} + \frac{3}{n^2} \Bigg(\frac{3B}{\lambda_k}\sqrt{n} + \frac{3C + 2K\Sigma}{\lambda_k}n^{1/4}+   K \frac{(B+C)^3}{\lambda_k}(\frac{1}{\Sigma^2} +8) +  \frac{8K \Sigma}{n^{3/2} \lambda_k}(B+C)^2( 1+ \frac{1}{\Sigma^2})\Bigg) \\
&\leq \frac{\Sigma}{n} + \frac{9B}{n^{3/2}\lambda_k} + \frac{8K\Sigma}{n^2} + \frac{3}{n^{7/4}\lambda_k} \Bigg( 3C + 2K\Sigma+  K (B+C)^3(1+\Sigma)(\frac{1}{\Sigma^2} +8)\Bigg) \\
&\leq \frac{\Sigma}{n} + \frac{9B}{n^{3/2}\lambda_k} + \frac{8K\Sigma}{n^2} + \frac{3}{n^{7/4}\lambda_k} \Bigg(K (B+C)^3(1+\Sigma)(\frac{1}{\Sigma^2} +13)\Bigg) \\
&\leq \frac{\Sigma}{n} + \frac{9B}{n^{3/2}\lambda_{\min}} + 3 K (B+C)^3(1+\Sigma)(\frac{1}{\Sigma^2} +21)\frac{1}{n^{7/4}\lambda_{\min}} 
\end{align*}
\end{small}
again since $B+C \geq 1$.

\noindent Finally, combining that with Lemma~\ref{l:some-lemma} gives us for the regret:

\begin{small}
\begin{equation*}
 R_n(\alg_B) \leq \frac{9B}{n^{3/2}\lambda_{\min}} + 3 K \frac{(B+C)^3}{n^{7/4}\lambda_{\min}}(\frac{1}{\Sigma^2} +21)(1 + \Sigma) + 2c_1n^2K\de (1 + \log(c_2/2nK\de)).
\end{equation*}
\end{small}

\noindent By taking $\de =  n^{-7/2}$ and recalling the bounds on $B$ and $C$ in \ref{app:tr}, we obtain:

\begin{small}
\begin{align*}
 R_n(\alg_B) &\leq  \frac{9B}{n^{3/2}\lambda_{\min}} + 3 K \frac{(B+C)^3}{n^{7/4}\lambda_{\min}}(\frac{1}{\Sigma^2} +21)(1 + \Sigma)+ 7c_1(c_2+1)K\log(n) n^{-3/2}\\
&\leq  \frac{76400 c_1(c_2+1) K^2\log(n)^2}{\lambda_{\min}n^{3/2}} + O\Big(\frac{\log(n)^6 K^7}{n^{7/4}\lambda_{\min}}\Big).
\end{align*}
\end{small}

\end{proof}


\section{Regret Bound for Gaussian Distributions}\label{s:b-results-gauss}

Here we report the proof of Lemma~\ref{l:gauss-regret} which implies that when the distributions of the arms are Gaussian, bounding the regret of the B-AS algorithm does not require upper-bounding the number of pulls $T_{k,n}$ (it can be bounded only by using a lower bound on the number of pulls).

Let $\{X_t\}_{t\geq 1}$ be a sequence of i.i.d.~random variables drawn from a Gaussian distribution $\N(\mu,\sigma^2)$. Write $\hat m_t=\frac{1}{t}\sum_{i=1}^t X_i$ and $\hat s_t^2=\frac{1}{t-1}\sum_{i=1}^t (X_i-\hat m_t)^2$ for the empirical mean and variance of the first $t$ samples. 

Before proving Lemma~\ref{l:gauss-regret}, we recall a property of the normal distribution (see e.g.,~\cite{bremaud1988}).
\begin{proposition}\label{p:normaldistr}
Let $X_1,\ldots,X_t$ be $t$ i.i.d.~Gaussian random variables. Then their empirical mean $\hat m_t=\frac{1}{t}\sum_{i=1}^t X_i$ and empirical variance $\hat s_t^2=\frac{1}{t-1}\sum_{i=1}^t (X_i-\hat m_t)^2$ are independent of each other.
\end{proposition}

Based only on the well-known $t=2$ case (i.e., that $X_1+X_2$ and
$|X_1-X_2|$ are independent), we can derive a somewhat stronger result that is
used in the proof of Lemma~\ref{l:gauss-regret}, showing that for Gaussian distributions, the
empirical mean $\hat m_t$ built on t i.i.d.~samples is independent from the
sequence of standard deviations $(\hat s_2,..., \hat s_t)$ (not only from $\hat s_t^2$).

We first derive a general result showing that for Gaussian distributions, the empirical mean $\hat m_t$ built on $t$ i.i.d. samples is independent from the sequence of standard deviations $\hat s_2,\dots,\hat s_t$.

\begin{lemma}\label{lem:gauss.law}
Let $\F_t$ be the $\sigma$-algebra generated by the sequence of random variables $\hat s_2,\dots,\hat s_t$. Then for all $t\geq 2$,
\beqan
\hat m_t \big| \F_t \sim \N\Big(\mu, \frac{\si^2}{t}\Big).
\eeqan
\end{lemma}

To prove Lemma~\ref{lem:gauss.law}, we need the following technical lemma:
\begin{lemma}\label{lem:expr:si}
We have
 \beqan
 \hat s_{t+1}^2 &=& \frac{t-1}{t} \hat s_{t}^2 + \frac{1}{t+1} (X_{t+1} - \hat m_t)^2.
 \eeqan
Note that this statement is deterministic, it holds for any process or sequence.
\end{lemma}
\begin{proof}

We have for $t\geq 2$
 \beqan
 \hat s_{t+1}^2 &=& \frac{1}{t} \sum_{i=1}^{t+1}(X_i - \hat m_{t+1})^2 \\
&=& \frac{1}{t} \sum_{i=1}^{t}(X_i - \hat m_{t+1} + \hat m_t - \hat m_t)^2 + \frac{1}{t} (X_{t+1} - \hat m_{t+1})^2\\
&=& \frac{1}{t} \sum_{i=1}^{t}(X_i - \hat m_t)^2 + \frac{1}{t} (X_{t+1} - \hat m_{t+1})^2 + (\hat m_t - \hat m_{t+1})^2\\
&=& \frac{1}{t} \sum_{i=1}^{t}(X_i - \hat m_t)^2 + \frac{t}{(t+1)^2} (X_{t+1} - \hat m_{t})^2 + \frac{1}{(t+1)^2} (X_{t+1} - \hat m_t)^2\\
&=& \frac{1}{t} \sum_{i=1}^{t}(X_i - \hat m_t)^2 + \frac{1}{t+1} (X_{t+1} - \hat m_{t})^2,
 \eeqan
which finishes the proof.
\end{proof}

From Lemma~\ref{lem:expr:si} we deduce by induction that for any $t\geq 2$ there exists a sequence of non-negative real numbers $\{a_{1,t}, a_{2,t},\dots,a_{t,t}\}$ such that 
\beqan
 \hat s_{t}^2 &=& a_{1,t} \hat s_2^2 + \sum_{i=2}^{t-1} a_{i,t}(X_{i+1} - \hat m_{i})^2.
\eeqan

\begin{proof}
We prove the statement by induction.

The base of the induction ($t=2$) is directly implied by the specific properties of Gaussian distributions (Proposition~\ref{p:normaldistr}). In fact, $\hat m_2$ is distributed as  $\N(\mu, \si^2 / 2)$ and $\hat m_2$ and $\hat s_2$ are independent.

Now we focus on the inductive step. For any $t\geq 2$, let $\G_t$ be the $\sigma$-algebra generated by the random variables $\hat s_2^2$ and $\{(X_{i+1} - \hat m_i)^2\}_{2\leq i\leq t-1}$. The recursive definition of the empirical variance in Lemma~\ref{lem:expr:si} immediately implies that the knowledge of $\{\hat s_2,\dots,\hat s_t\}$ is equivalent to the knowledge of $\hat s_2^2$ and $\{(X_{i+1} - \hat m_i)^2\}_{2\leq i\leq t-1}$ and thus $\F_t=\G_t$. We assume (inductive hypothesis)
\beq\label{eq:inductive.prop}
\hat m_t \big| \G_t \sim \N\Big(\mu, \frac{\si^2}{t}\Big),
\eeq
and we now show that (\ref{eq:inductive.prop}) also holds for $t+1$. Let $U = X_{t+1} - \hat m_t$ and $V = \hat m_{t+1} - \mu$. Note that $V$ can be written as $V = \frac{t}{t+1}(\hat m_{t} - \mu) + \frac{1}{t+1}(X_{t+1} - \mu)$. Since samples are i.i.d., $X_{t+1}$ is independent from $(X_1, \ldots, X_t)$ and
\beqan
X_{t+1} \big| \G_t \sim \N(\mu, \si^2)
\eeqan
and thus $X_{t+1}$ is also conditionally independent of $\hat m_t$ given $\G_t$. This implies that $X_{t+1}$ and $\hat  m_t$ are jointly Gaussian given $\G_t$ (two random variables that are Gaussian and independent are jointly Gaussian, see~\citep{bookgauss} or also \url{http://en.wikipedia.org/wiki/Multivariate_normal_distribution#Joint_normality}). This fact combined with the definition of $U$ and $V$ implies that $U$ and $V$ are conditionally jointly-Gaussian variables with zero conditional mean given $\G_t$ (they are jointly-Gaussian because they can be written as two independent linear combinations of the random variables $X_{t+1}-\mu$ and $\hat{m}_t-\mu$ given $\G_t$, see~\citep{bookgauss} or also \url{http://en.wikipedia.org/wiki/Multivariate_normal_distribution#Affine_transformation}). Furthermore, we can show that they are also conditionally uncorrelated given $\G_t$ since
\beqan
\E\Big[UV | \G_t\Big] &=&\E\Big[\Big( X_ {t+1} - \hat m_t\Big) \Big( \frac{1}{t+1}X_{t+1} + \frac{t}{t+1}\hat m_t - \mu\Big)  \Big| \G_t\Big]\\
&=& \E\Big[\Big( (X_{t+1}-\mu) - (\hat m_t - \mu)\Big) \Big( \frac{1}{t+1}(X_{t+1}-\mu) + \frac{t}{t+1}(\hat m_t - \mu)\Big)  \Big| \G_t \Big]\\
&=& \frac{1}{t+1}\sigma^2 - \frac{t}{t+1}\frac{\sigma^2}{t}=0.
\eeqan
As a result, $U$ and $V$ are conditionally independent given $\G_t$ and
\beqan
(\hat m_{t+1} - \mu) \big| \G_{t+1} 
= (\hat m_{t+1} - \mu) \big| \{ \G_{t}, (X_{t+1}-\hat m_t)^2 \}
= (\hat m_{t+1} - \mu) \big| \{ \G_{t}, U^2 \}
= V\big| \{ \G_{t}, U^2 \} = V|\G_t.
\eeqan
Since the induction assumption is verified, we know that $\E[V|\G_t] = 0$ and $\V[V|\G_t] = (\frac{t}{t+1})^2 \frac{\sigma^2}{t} + (\frac{1}{t+1})^2\sigma^2= \frac{\sigma^2}{t+1}$. Finally, we deduce that 
$$ \hat m_{t+1} \big| \G_{t+1} \sim \N\Big(\mu,\frac{\sigma^2}{t+1}\Big),$$
which concludes the proof since $\G_{t+1}=\F_{t+1}$.
\end{proof}

We now study an adaptive algorithm that computes the empirical average $\hat m_t$ and that at each time $t$ decides whether to stop collecting samples or not on the basis of the sequence of empirical standard deviations $\hat s_2,\dots,\hat s_t$ observed so far. Let $T\geq 2$ be a integer-valued random variable, which is a stopping time with respect to $\F_t$. This means that the decision of whether to stop at any time before $t+1$ (the event $\{T\leq t\}$) only depends on the previous empirical standard deviations $\hat s_2,\dots,\hat s_t$. From an immediate application of Lemma~\ref{lem:gauss.law} we obtain
\beqan
\E[(\hat m_{T}- \mu)^2]&=&\sum_{t\geq 2} \E[(\hat m_t-\mu)^2|T=t]\Prob[T=t] \\
&=& \sum_{t\geq 2} \E[\E[(\hat m_t-\mu)^2|\F_t, T=t]| T=t] \Prob[T=t] \\
&=& \sum_{t\geq 2} \E[\E[(\hat m_t-\mu)^2|\F_t] | T =t] \Prob[T=t] = \sum_{t\geq 2}  \frac{\sigma^2}{t} \Prob[T=t] 
= \sigma^2 \E\Big[\frac 1T\Big].
\eeqan

The previous result seamlessly extends to the general multi-armed bandit allocation strategies considered in Section~\ref{s:ch-algorithm} and~\ref{s:b-algorithm}.


\begin{proof}[Proof of Lemma~\ref{l:gauss-regret}]

Let us now consider algorithms CH-AS and B-AS. For any arm $k$, the event $\{T_{k,n}> t\}$ depends on the $\sigma$-algebra $\F_{k,t}$ (generated by the sequence of empirical variances of the first $t$ samples of arm $k$) and also on the ``environment'' $\mathcal{E}_{-k}$ (generated by all the samples of other arms). Since the samples of arm $k$ are independent from $\mathcal{E}_{-k}$, we deduce that by conditioning on $\mathcal{E}_{-k}$ Lemma~\ref{lem:gauss.law} still applies and
\beqan
\E[(\hat\mu_{k,n}-\mu)^2] = 
\E_{\mathcal{E}_{-k}} \big[ \E[(\hat\mu_{k,n}-\mu)^2|\mathcal{E}_{-k}]\big] = \sigma_k^2 \E_{\mathcal{E}_{-k}} \Big[  \E\Big[\frac 1{T_{k,n}} | \mathcal{E}_{-k} \Big] \Big]
= \sigma_k^2 \E\Big[\frac 1{T_{k,n}}\Big].
\eeqan

\end{proof}

We now report the proof of Theorem~\ref{thm:b-regret-gauss}.


\begin{proof}[Proof of Theorem~\ref{thm:b-regret-gauss}]
We recall Lemma~\ref{l:gauss-regret} and decompose the loss using the definition of $\xi = \xi_{K,n}^B(\de)$ in order to obtain
\begin{small}
\begin{equation*}
L_{k,n} = \var_k \E\Big[\frac{1}{T_{k,n}}\Big] = \var_k \E\Big[\frac{1}{T_{k,n}} \ind{\xi}\Big] + \var_k \E\Big[\frac{1}{T_{k,n}}  \ind{\xi^c}\Big].
\end{equation*}
\end{small}

\noindent From the bound in Equation~\eqref{eq:allowb.mod}, we have (since $n \geq 5K$)
\begin{small}
\begin{align}\label{eq:binxi}
\var_k \E\Big[\frac{1}{T_{k,n}} \ind{\xi}\Big]   &\leq  \max_{\xi}\Big[\frac{\var_k}{T_{k,n}}\Big] \nonumber\\
&\leq \frac{\Sigma}{n-2K} + \frac{B}{n^{1/2}(n-2K)} +\frac{C}{n^{3/4}(n-2K)}\nonumber\\
&\leq \frac{\Sigma}{n} + \frac{4K\Sigma}{n^2}+ \frac{2B}{n^{3/2}} +\frac{2C}{n^{7/4}}\nonumber\\
&\leq \frac{\Sigma}{n} + \frac{4K\Sigma}{n^2}+ \frac{12\times 10^3}{n^{3/2}} K^2 c_1(c_2+1)(\log n)^2 + \frac{14\times 10^3}{n^{7/4}} K^2 c_1(c_2+1)(\log n)^2 \nonumber\\
&\leq \frac{\Sigma}{n} + \frac{12.001\times 10^3}{n^{3/2}} K^2 c_1(c_2+1)(\log n)^2 + \frac{14\times 10^3}{n^{7/4}} K^2 c_1(c_2+1)(\log n)^2\nonumber\\
&\leq \frac{\Sigma}{n} + \frac{26.001\times 10^3}{n^{3/2}} K^2 c_1(c_2+1)(\log n)^2.
\end{align}
\end{small}
\noindent 
where we use the bounds on $B$ and $C$ in \ref{app:tr}. Using the fact that $\delta = n^{-7/2}$ and $T_{k,n}\geq 2$, and by Lemma~\ref{l:event-B-AS} that tells us $\Prob [\xi^c] \leq 2nK\delta$, we may write

\begin{small}
\begin{equation}\label{eq:boutxi.norm}
\var_k \E\Big[\frac{1}{T_{k,n}} \ind{\xi^c}\Big] \leq K \var_k n^{-5/2} \leq c_1c_2K n^{-5/2}.
\end{equation}
\end{small}

\noindent Finally, combining Equations~\ref{eq:binxi} and~\ref{eq:boutxi.norm}, and recalling the definition of regret, we have
\begin{small}
\begin{align}
R_n(\alg_B)  &\leq  \frac{26.001\times 10^3}{n^{3/2}}K^2c_1(c_2+1)(\log n)^2 + c_1c_2Kn^{-5/2}\\
&\leq \frac{26.002\times 10^3}{n^{3/2}}K^2 c_1(c_2+1)(\log n)^2 \nonumber\\
&\leq \frac{105\times 10^3 \bar \Sigma}{n^{3/2}}K^2(\log n)^2, \nonumber\\
\end{align}
\end{small}
since $c_1 = 2 \bar \Sigma$ and $c_2=1$.

\end{proof}

\end{document}

%% file: symbolsdef.tex
\newcommand{\beq}{\begin{equation}}
\newcommand{\eeq}{\end{equation}}

\newcommand{\beqa}{\begin{eqnarray}}
\newcommand{\eeqa}{\end{eqnarray}}

\newcommand{\beqan}{\begin{eqnarray*}}
\newcommand{\eeqan}{\end{eqnarray*}}

\newcommand{\E}{\mathbb{E}}
\newcommand{\1}{\mathbb{I}}

\newcommand{\invdelta}{(1/\delta)}

\newcommand{\alg}{\mathcal A}

\def\argmax{\mathop{\rm arg\,max}}

\newcommand{\narms}{K}

\newcommand{\distro}{\nu}

\newcommand{\hmu}{\hat{\mu}}

\newcommand{\var}{\sigma^2}
\newcommand{\hvar}{\hat{\sigma}^2}

\newcommand{\si}{\sigma}
\newcommand{\hsi}{\hat{\sigma}}

\newcommand{\F}{\mathcal{F}}

\newcommand{\eps}{\varepsilon}

\newcommand{\ind}[1]{\mathbb I\left\lbrace {#1} \right\rbrace}
\newcommand{\I}{\mathbb I}

\newcommand{\Prob}{\mathbb P}
\newcommand{\V}{\mathbb V}

\newcommand{\expectB}[2]{\mathbb E_{#1} \left[ {#2} \right]}

\newcommand{\N}{\mathcal N}

\newtheorem{assumption}{Assumption}

%% file: adapt_alloc_tech-report.bbl
\begin{thebibliography}{13}
\providecommand{\natexlab}[1]{#1}
\providecommand{\url}[1]{\texttt{#1}}
\expandafter\ifx\csname urlstyle\endcsname\relax
  \providecommand{\doi}[1]{doi: #1}\else
  \providecommand{\doi}{doi: \begingroup \urlstyle{rm}\Url}\fi

\bibitem[Antos et~al.(2010)Antos, Grover, and Szepesv\'{a}ri]{antos2010active}
Andr\'{a}s Antos, Varun Grover, and Csaba Szepesv\'{a}ri.
\newblock Active learning in heteroscedastic noise.
\newblock \emph{Theoretical Computer Science}, 411:\penalty0 2712--2728, June
  2010.

\bibitem[Audibert et~al.(2009)Audibert, Munos, and Szepesvari]{AudibertTCS09}
J-Y. Audibert, R.~Munos, and Cs. Szepesvari.
\newblock Exploration-exploitation trade-off using variance estimates in
  multi-armed bandits.
\newblock \emph{Theoretical Computer Science}, 410:\penalty0 1876--1902, 2009.

\bibitem[Audibert et~al.(2010)Audibert, Bubeck, and Munos]{audibert2010best}
J.-Y. Audibert, S.~Bubeck, and R.~Munos.
\newblock Best arm identification in multi-armed bandits.
\newblock In \emph{Proceedings of the Twenty-Third Annual Conference on
  Learning Theory (COLT'10)}, pages 41--53, 2010.

\bibitem[Br{\'e}maud(1988)]{bremaud1988}
P.~Br{\'e}maud.
\newblock \emph{An Introduction to Probabilistic Modeling}.
\newblock Springer, 1988.

\bibitem[Bubeck et~al.(2011)Bubeck, Munos, and Stoltz]{bubeck2011pure}
S\'{e}bastien Bubeck, R\'{e}mi Munos, and Gilles Stoltz.
\newblock Pure exploration in finitely-armed and continuous-armed bandits.
\newblock \emph{Theoretical Computer Science}, 412:\penalty0 1832--1852, April
  2011.
\newblock ISSN 0304-3975.

\bibitem[Castro et~al.(2005)Castro, Willett, and Nowak]{castro2005faster}
R.~Castro, R.~Willett, and R.~Nowak.
\newblock Faster rates in regression via active learning.
\newblock In \emph{Proceedings of Neural Information Processing Systems
  (NIPS)}, pages 179--186, 2005.

\bibitem[Chaudhuri and Mykland(1995)]{chaudhuri1995on-efficient}
P.~Chaudhuri and P.A. Mykland.
\newblock On efficient designing of nonlinear experiments.
\newblock \emph{Statistica Sinica}, 5:\penalty0 421--440, 1995.

\bibitem[Cohn et~al.(1996)Cohn, Ghahramani, and Jordan]{cohn1996active}
David~A. Cohn, Zoubin Ghahramani, and Michael~I. Jordan.
\newblock Active learning with statistical models.
\newblock \emph{J. Artif. Int. Res.}, 4:\penalty0 129--145, March 1996.
\newblock ISSN 1076-9757.

\bibitem[Eaton(1983)]{bookgauss}
Morris~L Eaton.
\newblock \emph{Multivariate statistics: a vector space approach}.
\newblock Wiley New York, 1983.

\bibitem[{\'E}tor{\'e} and Jourdain(2010)]{etore2010adaptive}
Pierre {\'E}tor{\'e} and Benjamin Jourdain.
\newblock Adaptive optimal allocation in stratified sampling methods.
\newblock \emph{Methodology and Computing in Applied Probability}, 12:\penalty0
  335--360, 2010.

\bibitem[Fedorov(1972)]{fedorov1972theory}
V.~Fedorov.
\newblock \emph{Theory of Optimal Experiments}.
\newblock Academic Press, 1972.

\bibitem[Hoeffding(1963)]{Hoeffding63PI}
Wassily Hoeffding.
\newblock Probability inequalities for sums of bounded random variables.
\newblock \emph{Journal of the American Statistical Association}, 58\penalty0
  (301):\penalty0 13--30, March 1963.
\newblock URL \url{http://www.jstor.org/stable/2282952?}

\bibitem[Maurer and Pontil(2009)]{maurer2009empirical}
A.~Maurer and M.~Pontil.
\newblock Empirical bernstein bounds and sample-variance penalization.
\newblock In \emph{Proceedings of the Twenty-Second Annual Conference on
  Learning Theory}, pages 115--124, 2009.

\end{thebibliography}
